\documentclass[letterpaper]{article} 
\usepackage{aaai24}  
\usepackage{times}  
\usepackage{helvet}  
\usepackage{courier}  
\usepackage[hyphens]{url}  
\usepackage{graphicx} 
\usepackage{natbib}  
\usepackage{caption} 
\usepackage{algorithm}
\usepackage{algorithmic}

\usepackage{newfloat}
\usepackage{listings}
\DeclareCaptionStyle{ruled}{labelfont=normalfont,labelsep=colon,strut=off} 
\floatstyle{ruled}
\newfloat{listing}{tb}{lst}{}
\floatname{listing}{Listing}
\usepackage{microtype}
\usepackage{graphicx}
\usepackage{subfigure}
\usepackage{dsfont}
\usepackage{booktabs} 
\usepackage{url}            
\usepackage{booktabs}       
\usepackage{amsfonts}       
\usepackage{nicefrac}       
\usepackage{microtype}      
\usepackage{xcolor}         
\usepackage{algorithm}
\usepackage{algorithmic}
\usepackage{graphicx}
\usepackage{booktabs}
\usepackage{longtable}
\usepackage{url}
\usepackage{amsmath}
\usepackage{amssymb}
\usepackage{mathtools}
\usepackage{amsthm}
\usepackage{float}
\theoremstyle{plain}
\newtheorem{theorem}{Theorem}[section]

\newtheorem{lemma}[theorem]{Lemma}

\theoremstyle{definition}

\theoremstyle{remark}

\title{CUDC: A Curiosity-Driven Unsupervised Data Collection Method with Adaptive Temporal Distances for Offline Reinforcement Learning}
\author{
    Chenyu Sun\textsuperscript{\rm 1,2,3}, Hangwei Qian\textsuperscript{\rm 4}, Chunyan Miao\textsuperscript{\rm 1,2,3}
}
\affiliations{
    \textsuperscript{\rm 1}Alibaba-NTU Singapore Joint Research Institute, Nanyang Technological University (NTU), Singapore\\
    \textsuperscript{\rm 2}School of Computer Science and Engineering, NTU, Singapore\\
    \textsuperscript{\rm 3}Joint NTU-UBC Research Centre of Excellence in Active Living for the Elderly (LILY), NTU, Singapore\\
    \textsuperscript{\rm 4}Centre for Frontier AI Research (CFAR), Agency for Science, Technology and Research (A*STAR), Singapore\\

    chenyu002@e.ntu.edu.sg, qian\_hangwei@cfar.a-star.edu.sg, ascymiao@ntu.edu.sg
}

\usepackage{bibentry}

\begin{document}

\maketitle

\begin{abstract}
Offline reinforcement learning (RL) aims to learn an effective policy from a pre-collected dataset.
Most existing works are to develop sophisticated learning algorithms, with less emphasis on improving the data collection process. Moreover, it is even challenging to extend the single-task setting and collect a task-agnostic dataset that allows an agent to perform multiple downstream tasks. In this paper, we propose a \textbf{C}uriosity-driven \textbf{U}nsupervised \textbf{D}ata \textbf{C}ollection (CUDC) method to expand feature space using adaptive temporal distances for task-agnostic data collection 
in multi-task offline RL. To achieve this, 
CUDC estimates the probability of the $k$-step future states being reachable from the current states, and adapts how many steps into the future that the dynamics model should predict.
With this adaptive reachability mechanism in place, the feature representation can be diversified, and the agent can navigate itself to collect higher-quality data with curiosity. Empirically, CUDC surpasses existing unsupervised methods in efficiency and learning performance in various downstream offline RL tasks of the DeepMind control suite.
\end{abstract}

\section{Introduction}

    Deep reinforcement learning has achieved remarkable breakthroughs in various fields, such as games, robotics, and navigation in virtual environments \citep{kiran2021deep, singh2022reinforcement, ijcai2022p478}. 
    However, real-time interaction with the environment under online RL settings may not always be feasible due to cost, safety, or ethical concerns \citep{kiran2021deep, singh2022reinforcement}. As a result, offline RL has gained popularity in recent years to cope with limited interactions, where agents learn a policy exclusively from a previously-collected dataset. The popular offline RL benchmarks such as D4RL \cite{fu2020d4rl} and RL Unplugged \cite{gulcehre2020rl} combine data from supervised online RL training runs with expert demonstrations, exploratory agents, and hand-coded controllers. However, collecting expert data can be time-consuming and expensive, and it may not always be available. In such cases, unsupervised methods, such as those described by ExORL \cite{yarats2022don}, can be used to collect data as a distinct contribution for offline RL \cite{prudencio2022survey}. These methods aim to explore the environment and learn from the intrinsic rewards generated by the agent, without the need for supervision, to collect diverse data.

	    \begin{figure*}[t]
    \begin{center}
        \vspace{-0.5mm}
        \includegraphics[width=0.98\textwidth]{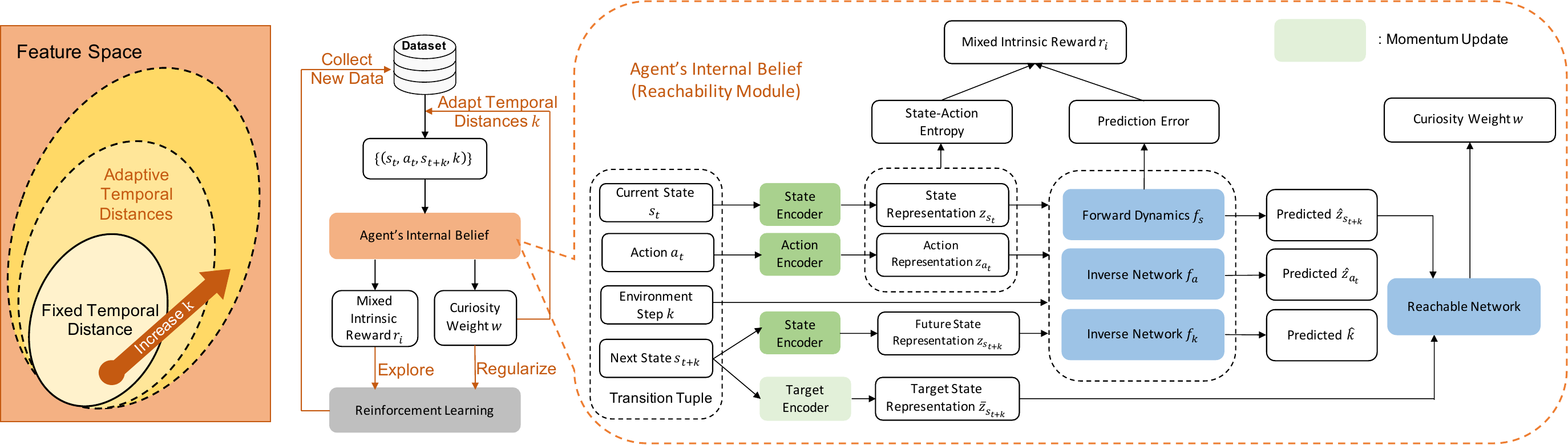}
    \end{center}
    \vspace{-3.5mm}
    \caption{\small 
    Curiosity-driven Unsupervised Data Collection (CUDC): The left diagram depicts the relationship between fixing (existing works) and adapting (CUDC) temporal distance in feature space. The middle diagram outlines the CUDC framework, measuring reachability between $k$-step future and current states using the agent's internal belief. It generates mixed intrinsic rewards for diverse exploration and curiosity weight to adapt temporal distance, regulating the RL backbone. This process continues until the capacity is reached. The right diagram illustrates how the agent assesses and updates its internal belief regarding the probability of $k$-step future states being reachable from current states.
    }
    \label{workflow}
    \vskip -0.15in
\end{figure*}  \normalsize
 
    Despite the popularity of offline RL, existing works have mainly focused on model-centric practices, continually developing new algorithms \citep{kumar2020conservative, kumar2021should}. These algorithms are typically evaluated on the same task for which the dataset was collected, and the learned policy can be pessimistic in out-of-distribution states and actions, leading to poor generalization in unseen downstream tasks. Recently, data-centric approaches have become emerging, emphasizing the importance of training data quality over algorithmic advances \citep{ motamedi2021data, patel2022advances}. To improve training data quality, researchers have explored selecting the most critical samples or re-weighting \citep{wu2021uncertainty} all samples in the offline RL algorithms. However, these methods are restricted to a single training data distribution and cannot be applied to multi-task settings with distribution shifts.
    To address this challenge, we propose to improve the data collection process directly through feature space expansion, where the distributions naturally span during diverse exploration. This approach is applicable to the multi-task setting, enabling us to obtain more diverse and high-quality data for offline RL.
    
    Upon analyzing the current challenges faced in offline RL, the benchmark ExORL \citep{yarats2022don} has shown that unsupervised RL methods are more effective than supervised methods in collecting datasets that allow the vanilla off-policy RL algorithm to learn and acquire different skills as an offline RL agent. However, upon further examination of existing methods, we discovered that they rely on a fixed temporal distance $k$ between current and future states during data collection. This practice is sub-optimal and restricts the diversity of the learned feature representation, as illustrated in Figure~\ref{workflow}~(left). To address this limitation, we propose to adapt the temporal distance as a simple yet effective way to enhance the feature representation, as it has a direct connection with the feature space.
    
    To facilitate adaptation, exploiting reachability to more distant future states is desired. Reachability-based methods in RL aim to learn safe and efficient policies by considering reachable states under the current policy or value function \citep{savinov2018episodic, pere2018unsupervised, ivanovic2019barc, yu2022reachability}, but these approaches are not directly applicable. For example, \citet{savinov2018episodic} only considers binary reachability, and extensively compares to stored embeddings in memory. Additionally, the reachability in goal space exploration \cite{pere2018unsupervised} often requires kernel density estimation, which can increase computational cost substantially. Different from these, we propose a \textbf{C}uriosity-driven \textbf{U}nsupervised \textbf{D}ata \textbf{C}ollection (CUDC) method with a novel reachability module. Inspired by the fact that human curiosity can foster learning and is driven by novel knowledge beyond one's perception \citep{markey2014curiosity, sun2022psychological,sun2022cd}, CUDC facilitates data collection curiously without any task-specific reward. In particular, the reachability module estimates the probability of a $k$-step future state being reachable from the current state, with no episodic memory or feature space density modeling required. This module enables the agent to adaptively determine how many steps into the future that the dynamics model should predict, allowing for an enhanced feature representation to be learned. Compared with the existing unsupervised methods, it refrains from learning a fixed feature space. With this enhanced representation, CUDC utilizes a mixed intrinsic reward that encourages the agent to curiously explore meaningful state-action spaces and under-learned states. As a result, the collected dataset can lead to improved computational efficiency, sample efficiency, and learning performances in various downstream offline RL tasks.

	
	Our contributions can be summarized as follows. 
	1) We are the first to introduce reachability for improving data collection in offline RL, which is defined in a more efficient way and can enable the agent to navigate curiosity-driven learning coherently.
	2) We point out a common drawback of fixing the temporal distance in existing approaches, and empirically show that adapting the temporal distance in the reachability analysis can enhance feature representation by expanding the feature space.
	3) With the enhanced representations, CUDC additionally incentivizes the agent to explore diverse state-action space as well as the under-learned states with high prediction errors through a mixed intrinsic reward and regularization. 
	4) Under the ExORL benchmark setting \citep{yarats2022don}, CUDC outperforms other unsupervised methods when collecting the task-agnostic dataset that can be used for offline learning in multiple downstream tasks from the DeepMind control suite~\citep{tassa2018deepmind}.
	

	\section{Related Works}
	\paragraph{Reachability in RL} \citet{savinov2018episodic} devised a reachability network to estimate how many environment steps to take for reaching a particular state. It intrinsically rewards the agent to explore the state that is unreachable 
    from other states in memory. However, this approach only considers the binary case of reachability, potentially being inefficient when comparing with stored states. In goal exploration tasks, \citet{pere2018unsupervised} defined the reachability of a goal with an estimated density and proposed to sample increasingly difficult goals to reach during exploration. While this approach can learn the goal space in an unsupervised manner, its sampling process requires a kernel density estimator, which can substantially increase computational cost. Following the similar idea, BARC \citep{ivanovic2019barc} adapts the initial state distribution gradually from easy-to-reach to challenging-to-reach goals with physical priors in hard robotic control tasks. 
    Recently, RCRL \citep{yu2022reachability} shows that leveraging reachability analysis can help learn an optimal safe policy by expanding the limited conservative feasible set to the largest feasible set of the state space. Different from these works, CUDC is efficient and easy to implement, as it directly adapts the temporal distance to perform increasingly challenging reachability analysis without extensive comparisons, kernel density estimation or physical priors.
	
	\paragraph{Curiosity-Driven RL}
	Curiosity-driven RL is essential for encouraging agents to explore tasks in a human-like manner, especially when task-specific rewards are sparse or absent \citep{sun2022psychological}. The main approach to curiosity-driven RL involves incorporating intrinsic rewards that motivate agents to explore based on different aspects of the state, including novelty, entropy \citep{seo2021state, liu2021behavior}, reachability \citep{savinov2018episodic}, prediction errors \citep{pathak2017curiosity}, complexity \citep{campero2020learning}, and uncertainty \citep{pathak2019self}. Another approach is to prioritize experience replay towards under-explored states \citep{jiang2021prioritized}. Curiosity can also be used to explore other components of RL, as seen in CCLF \citep{ijcai2022p478}. CUDC is the first method to curiously adapt the temporal distance to explore more distant future states in offline RL, which enhances the learned representation space with increasingly challenging prediction. In addition, CUDC also regularizes Q-learning with a curiosity weight as the sample importance to focus more on under-learned tuples. 
	

	\paragraph{Unsupervised Data Collection}
	The ExORL benchmark \citep{yarats2022don} evaluates 9 unsupervised data collection algorithms, demonstrating superiority over supervised methods for multi-task offline learning. These unsupervised methods include knowledge-driven models like ICM \citep{pathak2017curiosity}, Disagreement \citep{pathak2019self}, and RND \citep{burda2019exploration}, which encourage exploration by maximizing prediction errors. Data-driven models like APT \citep{liu2021behavior} and ProtoRL \citep{yarats2021reinforcement} incentivize agents to uniformly explore the entire state space. Competence-based models like DIAYN \citep{eysenbach2018diversity}, SMM \citep{lee2019efficient}, and APS \citep{liu2021aps} encourage agents to learn diverse skills by leveraging prior information. However, all of these methods were originally designed for online pretraining and fine-tuning \citep{laskin2021urlb} and are not tailored for data collection. In contrast, CUDC is a novel method that gradually expands the feature space by exploiting reachability into more distant future states, rather than using a fixed temporal distance. 
    Additionally, CUDC exploits importance weights to focus more on under-learned tuples, which is not considered in Explore2Offline \citep{lambert2022challenges}, another recent method that leverages intrinsic model predictive control for simulating trajectories.

	\section{Curiosity-Driven Unsupervised Data Collection (CUDC)}
	\subsection{Problem Setting}
	We consider the problem of multi-task offline learning, which consists of three main steps: data collection, reward relabeling, and downstream offline learning, as described in both ExORL \citep{yarats2022don} and Explore2Offline \citep{lambert2022challenges}. In the data collection phase, the exploratory agent (data collector) has access to a Markov Decision Process (MDP) environment with a state $s \in \mathcal{S}$, an action $a \in \mathcal{A}$ based on a policy $\pi(s)$, a transition probability $p(s' | s, a)$ mapping from the current state $s$ and action $a$ to the next state $s'$, a reward $r$, and a discount factor $\gamma \in [0, 1)$ weighting future rewards. The exploratory agent collects a dataset $\mathcal{D}$ of unlabeled tuples $(s, a, s')$ by interacting with the environment. 
	The second phase is to relabel the collected dataset $\mathcal{D}$ using the given reward function $r_{\tau}(s,a)$ about the downstream task $\tau$ for each tuple. It transfers information from task-agnostic exploration to downstream tasks. 
	The last step is to perform multiple downstream tasks with an offline RL agent on the labeled dataset, without interacting with the environment to collect additional experiences. 
	In this paper, we focus on the most challenging part of this problem, which is the task-agnostic data collection and we evaluate the quality of the collected dataset $\mathcal{D}$ in multiple downstream tasks.

	\subsection{Framework Overview}

As shown in Figure~\ref{workflow} (mid), we propose a \textbf{C}uriosity-driven \textbf{U}nsupervised \textbf{D}ata \textbf{C}ollection (CUDC) method, which employs DDPG \citep{lillicrap2015continuous} as the base RL algorithm for the exploratory agent. To encourage diverse exploration, we introduce a novel reachability module, illustrated in Figure~\ref{workflow} (right), that calculates the likelihood of reaching a future state $k$ steps ahead of the current state. With this module in place, the exploratory agent can be encouraged to diversely explore by a mixed intrinsic reward, and meanwhile regularize the critic-actor update to prioritize under-learned tuples. Most importantly, the temporal distance of $k$-step between current and future states is adaptively increased to incorporate the dynamics information in the learned feature representation. This adaptation results in a more diverse exploration and improved data collection quality. Further details are presented in Algorithm~\ref{algorithm}.


	\subsection{The Reachability Module}
	In ExORL~\citep{yarats2022don}, existing unsupervised methods are limited by fixing the temporal distance $k=3$ between current and future states, as illustrated in Figure~\ref{workflow}~(left). To overcome this limitation and expand the feature space for improved representation learning, an intuitive approach is to employ reachability analysis for adaptive adjustment of $k$. However, existing reachability implementations are not desired due to limited binary classification of reachable states \citep{savinov2018episodic} or their reliance on costly density estimation of goal space \citep{pere2018unsupervised}. To address these issues, we propose a self-supervised reachability estimation method in CUDC, which estimates the probability of a $k$-step future state $s_{t_i+k}$ being reachable from the current state $s_{t_i}$ without requiring expensive density estimation or manual labeling.
    Consequently, our method can effectively enhance feature representation by expanding the feature space through an adaptive $k$-step. This approach has also been demonstrated to be effective in other works on reachability, such as constrained RL \citep{yu2022reachability} and robotics \citep{ivanovic2019barc}.
	
	
	Given a batch of unlabeled tuples $(s_{t_i}, a_{t_i}, s_{t_i+k}, k)^n_{i=1}$, existing methods in ExORL benchmark \citep{yarats2022don} simply fix the temporal distance $k=3$ throughout the data collection. In contrast, CUDC considers $k$ as a parameter and incorporates it explicitly into the tuples. We start by encoding the state features $z_{s_{t_i}}=\phi_s(s_{t_i})$, $z_{s_{{t_i}+k}}=\phi_s(s_{{t_i}+k})$, and the action feature $z_{a_{t_i}}=\phi_a(a_{t_i})$ using a state encoder $\phi_s(\cdot)$ and an action encoder $\phi_a(\cdot)$. We then perform one-hot encoding for the temporal distance $k$. To enable reachability analysis, we construct a forward dynamic network $\hat{z}_{s_{{t_i}+k}}=f_s(z_{s_{t_i}}, z_{a_{t_i}}, k; \theta_s)$ that takes as input $z_{s_{t_i}}$, $z_{a_{t_i}}$, and the encoded $k$ to predict the future state feature $\hat{z}_{s{{t_i}+k}}$, fully utilizing dynamics information. The network can be trained by minimizing the $l_2$ norm loss $||z_{s_{{t_i}+k}} -\hat{z}_{s_{{t_i}+k}}||_2$.
	
	To quantify the reachability, CUDC enforces $\hat{z}_{s_{{t_i}+k}}$ to match with its own $z_{s_{{t_i}+k}}$ as much as possible, while keeping apart from the other future states within the same batch. This contrastive intuition is that each future state should be most reachable from its own current state, and it can quantify the reachability in a simple and efficient way. Self-supervised contrastive learning has been shown to be capable of learning rich representations with more semantic latents in RL \citep{srinivas2020curl,liu2021behavior}, and CUDC follows this intuition to estimate the probability $l_i$ of $s_{{t_i}+k}$ being reachable from $s_{{t_i}}$ by:
    \begin{equation}
    \label{contrastive_loss}
    \scalebox{1}{$
    l_i=\frac{\text{sim}(\hat{z}_{s_{t_i+k}},{m}_{s_{t_i+k}})}{\text{sim}(\hat{z}_{s_{t_i+k}},{m}_{s_{t_i+k}})+\sum_{j=1, j\neq i}^{n}\text{sim}(\hat{z}_{s_{t_i+k}},{m}_{s_{t_j+k}})}$},
    \end{equation}
	where $\text{sim}(a,b)=\exp(h(a)^TW \bar{h}(b))$,  $n$ is the batch size, $h(\cdot)$ is a deterministic projection function, $W$ is a hidden weight to compute the similarity between the two projections, and $\bar{h}(\cdot)$ as well as ${m}(\cdot)$ are respectively the momentum-based moving average of the projection and state feature to ensure consistency and stability \citep{he2020momentum}. The reachability network is updated by minimizing the contrastive loss function $\mathcal{L}_{\text{reach}} = -\sum_{i=1}^{n} \log l_i$ in a self-supervised manner, without manual labeling.

 To further improve representation learning, the reachability module includes two inverse models for predicting action feature $\hat{z}_{a_{t_i}}$ and temporal distance $\hat{k}$. Similar to ICM \citep{pathak2017curiosity} and Disagreement \citep{pathak2019self}, we define $\hat{z}_{a_{t_i}}=f_a(z_{s_{t_i}}, z_{s_{{t_i}+k}}, k ; \theta_a)$ with a backward loss of $||z_{a_{t_i}}-\hat{z}_{a_{t_i}}||_2$. This loss ensures that the encoded features are robust to environment variations that are uncontrollable by the agent. For the inverse model of the $k$-step, $\hat{k}=f_k(z_{s_{t_i}}, z_{s_{{t_i}+k}} ; \theta_k)$ characterizes the prediction with a distribution $\mathbb{P}(k)$. The inverse model is updated through a cross-entropy loss, which enables the encoders to capture the dynamics information in the encoded features.
	
	By updating its internal belief in a self-supervised way, the agent can learn without the expensive labeling required in supervised learning. Additionally, the proposed reachability module allows the $k$-step temporal distance to adapt during learning, rather than relying on a fixed value in many existing unsupervised methods. This adaptability is important, as the feature representations of both states and actions become more informative and robust when adjusting the temporal distance of $k$-step.

	The reachability module also computes a curiosity weight $w_{i}$ for each tuple $i$ as $w_{i} = 1 - l_i \in [0,1]$, where $l_i$ is the contrastive loss defined in Equation \ref{contrastive_loss}.
    Intuitively, a large value of $w_{i}$ means that the agent does not believe the true future state is reachable from the current state, which induces high curiosity due to the conflict with current internal belief. It further indicates that this under-learned transition tuple contains novel information, and the encoders are not capable of extracting meaningful features yet. With this reachability module in place, we can seamlessly enable the agent to perform the task-agnostic dataset collection in a curious manner, which shall be illustrated in the next subsection.

	\subsection{Curiosity-Driven Learning} 
	To clarify, prior works on reachability such as \citep{savinov2018episodic} only incorporate reachability as an intrinsic reward to encourage diverse exploration. In contrast, our proposed CUDC leverages reachability in multiple stages of learning to promote curiosity-driven learning coherently. Firstly, it adapts the temporal distance, i.e. $k$-step, to expand the feature space and enhance feature representation with the prediction of future states. Secondly, it incorporates a mixed intrinsic reward to encourage effective exploration in under-learned state-action space with the enhanced representation. Lastly, it regularizes the critic-actor update for the backbone DDPG algorithm by utilizing the curiosity weights to focus more on under-learned tuples. Unlike the eight existing methods evaluated in ExORL that only utilize intrinsic rewards as curiosity, our CUDC extends curiosity-driven learning to different RL components, improving task-agnostic data collection coherently.

	\subsubsection{Enhance Feature Representation with Adaptive Temporal Distances}
	It is worth noting that the eight methods evaluated in ExORL limit the autonomy of the feature space by requiring the agent to reach future states exactly three steps away, i.e., $(s_{t_i}, a_{t_i}, s_{t_i+3})^n_{i=1}$. Recent online pre-training RL methods, such as SPR \citep{schwarzer2020data} and SGI \citep{schwarzer2021pretraining}, predict the agent's own latent state representations multiple steps into the future, improving sample efficiency. However, these methods require iterative predictions by calling the forward dynamic network $k$ times. In contrast, our proposed CUDC enables automatic adjustment of the temporal distance $k$ and performs $k$-step future state estimation directly, without substantially increasing computational complexity. The key idea is to keep the reachability estimation increasingly challenging with an adaptive $k$-step, thereby expanding the feature space to learn more meaningful reachability information. 
	
	In our approach, we dynamically adjust $k$ to impose more challenging reachability predictions, by leveraging the agent's level of curiosity. Specifically, we increase $k$ by 1 if the agent's curiosity level is low in the current reachability analysis, and we define a threshold $C_w$ for low curiosity and a threshold $C_k$ for the proportion of tuples with low curiosity. Thus, the agent adapts $k$ when the average value of $w_i$ is below $C_w$ for more than $C_k$ of the tuples in the batch, as represented by:
    \begin{equation}
    \scalebox{1}{$
        \frac{1}{n} \sum_i^n \mathds{1}_{\mathrm{w_i < C_w}} > C_k.$}
    \end{equation}
    The rationale behind this approach is that when the agent can estimate the current $k$-step reachability well for the majority of tuples in the batch, it should be encouraged to explore further. By expanding the feature space to learn the dynamics of more distant future states, the feature representation can be enhanced, leading to more informative and diverse task-agnostic data collection.
    It is worth noting that there are other possible ways to vary the $k$-step, such as by sampling from a probabilistic distribution. To validate the effectiveness of our proposed curiosity-driven method compared to other sampling-based methods, we conduct an ablation study in Section~\ref{experiments}.
	

	\subsubsection{Incorporate a Mixed Intrinsic Reward}
	CUDC utilizes a mixed intrinsic reward that combines state-action entropy and prediction error of future states. While previous methods like APT \citep{liu2021behavior} and RE3 \citep{seo2021state} have demonstrated that particle-based k-nearest neighbors state entropy can encourage agents to explore the state space more uniformly, we believe that exploration should not be limited to the state space alone, but should also extend to the action space. To achieve this, CUDC expands state embedding to state-action embedding and shows that entropy maximization can be applied to the k-nearest neighbor entropy estimation in the state-action representation space in Lemma \ref{lemma:entropy}. This approach encourages the agent to explore both the state and action spaces more diversely, leading to more effective and informative data collection.
	\begin{lemma}
		\label{lemma:entropy}
		Let $u=(z_s, z_a)$ represent the state-action representation. The particle-based entropy $\mathcal{H}(u)$ is proportional to a K-nearest neighbor (K-NN) distance,
		\begin{equation*}
		\mathcal{H}(u) \propto \sum_{i=1}^{n} \log||u_i - u_i^{\text{K-NN}}||_2.
		\end{equation*}
	\end{lemma}
 \begin{proof}
 A proof is provided in Appendix \ref{sec:proof}.
\end{proof}
	We build on the idea of treating each tuple as a particle \citep{liu2021behavior, seo2021state} and propose an intrinsic reward to estimate particle-based entropy, defined as  $r_{\mathcal{H}}(s_{t_i},a_{t_i})= \log (\frac{1}{N_K}\sum ||u_i - u_i^{\text{K-NN}}||_2+1)$, where $u_i=(\phi_s(s_{t_i}), (\phi_a(a_{t_i}))$, $N_K$  is the number of K-NN, and $\phi_s$ and $\phi_a$ are state and action encoders respectively. Since the encoded features are constantly updated to capture the dynamics of more distant future states in the reachability module, the proposed $r_{\mathcal{H}}$ promotes diverse state-action space exploration. This is consistent with the entropy maximization principle \citep{singh2003nearest} and has been shown to be effective in the state space using the state-of-the-art off-policy RL algorithm SAC \citep{haarnoja2018soft}.
    
	Additionally, we integrate prediction error of future states as another component of the intrinsic reward to incentivize the agent to explore surprising states beyond its expectations \citep{pathak2017curiosity, burda2018large}. Specifically, we use $r_{{\mathcal{E}}}(s_{t_i},a_{t_i})=||z_{s_{t_i+k}} -\hat{z}_{s_{t_i+k}}||_2$, where the reachability module is conveniently re-used without additional networks. Finally, the mixed intrinsic reward in CUDC is given by
	\begin{equation}
	\label{intrinsic_reward}
	r_i(s_{t_i},a_{t_i})=r_{\mathcal{H}}(s_{t_i},a_{t_i}) + \alpha r_{\mathcal{E}}(s_{t_i},a_{t_i}) + \beta,
	\end{equation}
	where $\alpha$ prioritizes under-learned state exploration and $\beta$ is a constant for numerical stability.

          \begin{algorithm}[t]
		\caption{Implementation of the proposed CUDC}
		\label{algorithm}
            \small
		\textbf{Initialize} parameters of encoders $\phi_s$ and $\phi_a$, forward dynamic $f_s$, inverse models $f_a$ and $f_k$, projection~$h$, critic $Q$, policy $\pi$, hidden weight $W$,  temporal distance $k$, batch size $n$, and an empty dataset~$\mathcal{D}=\emptyset$
  
		\begin{algorithmic}[0] 
			\FOR{each time step $t$}
			\STATE // COLLECT TRANSITIONS
			\STATE Interact with the environment using the policy $a_t \sim \pi(s_t)$ and observe $s_{t+1}$
			\STATE $\mathcal{D} \cup (s_t, a_t, s_{t+1}) \to \mathcal{D}$ 
			\STATE // UPDATE INTERNAL BELIEF
			\STATE Sample a minibatch $\{(s_{t_i}, a_{t_i}, s_{t_i+k}, k)\}_{i=1}^{n}\sim \mathcal{D}$
			\FOR{each tuple $i$ in the minibatch} 
			\STATE Encode the state and action, and predict the $t_{i}+k$'s future state feature $\hat{z}_{s_{t_i+k}}$
			\STATE Evaluate the curiosity weight $w_{i} =  1 - l_i$ by Eq.~(\ref{contrastive_loss})
			\STATE Compute the intrinsic reward $r_i$ using Eq.~(\ref{intrinsic_reward})
			\ENDFOR
			\STATE Update the internal belief of the reachability module
			\STATE //ADAPT THE K-STEP TO PREDICT
			\IF{$\frac{1}{n}\sum_i^n \mathds{1}_\mathrm{w_i < C_w} > C_k$}
			
			\STATE Increase the temporal distance by $k=k+1$
			\ENDIF
			\STATE //REGULARIZE CRITIC-ACTOR UPDATE
			\STATE Update the critic $Q$ with regularization by Eq.~(\ref{Q_learning})
			\STATE Update the actor $\pi$ with regularization
			\STATE Perform the momentum update for $\bar{h}$ and ${m}$
			\ENDFOR
		\end{algorithmic}
	\end{algorithm}
 \normalsize	
	\subsubsection{Regularize the critic-actor update}
Furthermore, CUDC utilizes the curiosity weight $w_i$ to adaptively regularize the backbone DDPG algorithm, allowing it to focus more on under-learned tuples. The weight $w=(w_1, w_2, \cdots, w_n)$ quantitatively characterizes the curiosity weight of each transition tuple, which can be used to determine sample importance and regularize both critic and actor updates. Therefore, the Q-learning in DDPG can be performed by minimizing the following objective,
	\begin{equation}
	\label{Q_learning}
        \scalebox{0.85}{$\mathbb{E}_{\cdot \sim \mathcal{D}}\left[ w \left(Q(s_t, a_t)-(r_i(s_t,a_t)+\gamma Q_{\text{target}}(s_{t+k}, \pi(s_{t+k}))) \right)^2\right]$}.
	\end{equation}
	Meanwhile, the policy can be updated by maximizing 
	$
	\mathbb{E}_{\cdot \sim \mathcal{D}} \left[ w Q(s_t,\pi (s_t) )\right]
	$. In this way, CUDC enables the agent to adapt its learning process in a self-supervised manner by using the conceptualized curiosity to exploit sample importance.

     \begin{figure*}[t]
    \centering
    \vspace{-1mm}
    \includegraphics[width=1\textwidth]{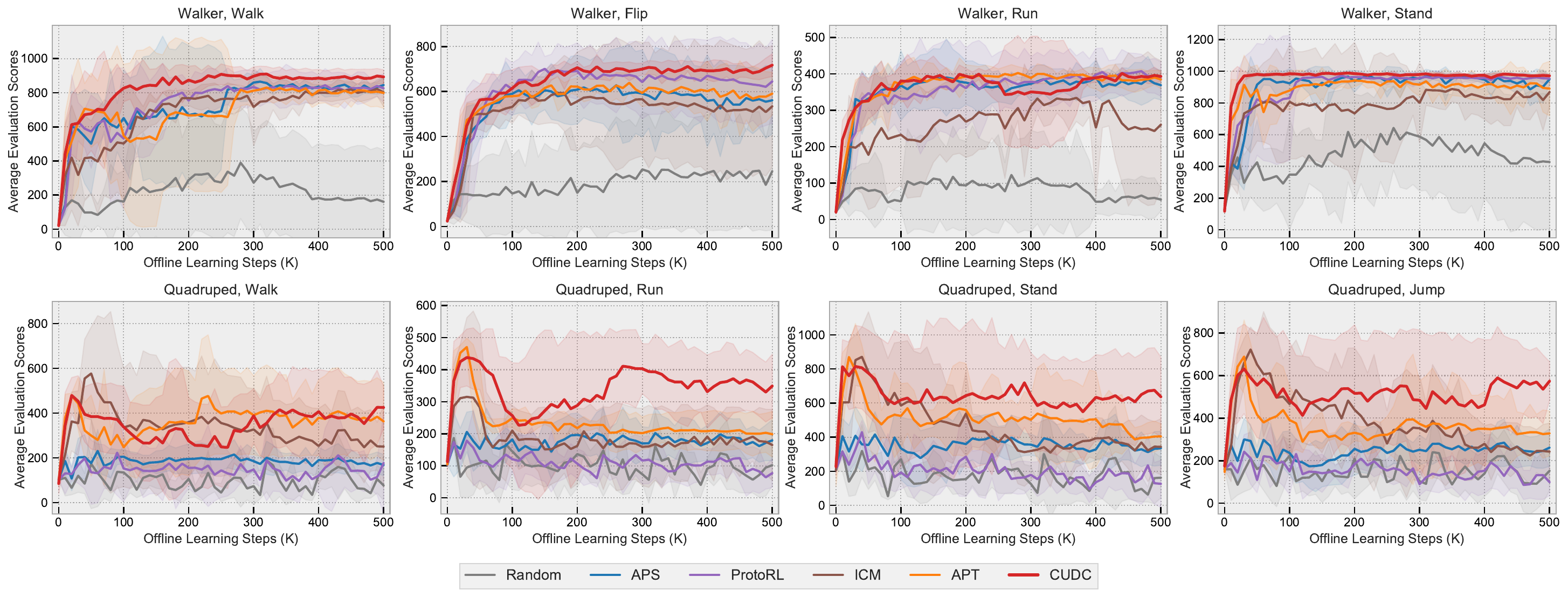}
    \vspace{-6mm}
    \caption{\small Learning curves of the offline RL agent on the task-agnostic dataset collected by different methods. The proposed CUDC demonstrates the superior capability of improving the computational efficiency and learning performances of the offline RL agent.}
    \label{fig:mian_two_curves}
    \vskip -0.12in
    \end{figure*}
    \normalsize

	\section{Experiments}
 \label{experiments}
 
	\paragraph{Environments}
	We evaluated on a set of challenging continuous control tasks with state observations, drawn from the DeepMind control suite \citep{tassa2018deepmind}. The suite contains 12 downstream tasks, organized into three main domains: Walker, Quadruped, and Jaco Arm. Walker is a controllable entity with locomotion-related balancing controls, where it can learn to walk, run, flip, and stand. Quadruped is a passively stable body in a more challenging 3D environment, which requires learning various locomotion skills such as walking, running, standing, and jumping. Jaco Arm is a six-degree-of-freedom robotic arm with a three-finger gripper for object manipulation, where the downstream tasks require it to reach different positions. Note that the PointMass Maze task is not included, as most baseline methods in ExORL have already demonstrated excellent performances on it.

	\paragraph{Baseline Models}
	We compare CUDC against state-of-the-art unsupervised methods across all three categories as benchmarked in ExORL, i.e., a knowledge-driven baseline of ICM \citep{pathak2017curiosity}, data-driven baselines of APT \citep{liu2021behavior} and ProtoRL \citep{yarats2021reinforcement}, and a competence-driven baseline of APS \citep{liu2021aps}. Meanwhile, a random data collector is also included, which collects the data by performing randomly sampled actions. The other four methods discussed in ExORL are excluded since their performance are less competitive.
	We use the same hyperparameters and model architecture as reported in ExORL to ensure a fair comparison. To demonstrate that all proposed components play important roles in the performance, we also compare four versions of CUDC. $\text{CUDC}_\text{vary}^\text{ICM}$ and $\text{CUDC}_\text{vary}^\text{APT}$: adapting the temporal distance of $k$-step by the intrinsic rewards based on the original ICM and APT methods. $\text{CUDC}_\text{reward}$: extending to state-action entropy with a mixed intrinsic reward based on $\text{CUDC}_\text{vary}^\text{APT}$. $\text{CUDC}_\text{reach}$: adding the full reachability module without regularization based on $\text{CUDC}_\text{reward}$. The detailed implementation and differences from the full model are summarized in Appendix~\ref{appendix:hyper}.
	
	\paragraph{Model Training and Evaluation}
	To ensure model stability during learning, we have restricted the temporal distance $k$ to be increased from 3 to 6 and have set upper and lower bounds for the regularization weights to guarantee stability.
	For further details regarding the network implementation and hyperparameter setting of the proposed CUDC, readers can refer to Appendix~\ref{appendix:hyper}. During data collection, all methods have been trained using a DDPG \citep{lillicrap2015continuous} agent as the backbone to ensure fairness.  They have interacted with 3 domain environments in the absence of extrinsic rewards for 1M steps. For the main results, a total of 90 datasets (6 algorithms $\times$ 3 main tasks $\times$ 5 seeds) have been collected. Afterwards, relabeling has been performed for each downstream task. During the evaluation, a TD3 \citep{fujimoto2018addressing} agent learns offline from each relabeled dataset for 500K steps. We report the performance score at 100K steps for computational efficiency and at 500K steps for learning performance.

 \paragraph{Main Results on 12 Downstream Tasks}
	Figure~\ref{fig:mian_two_curves} indicates that ProtoRL performs well in the Walker domain but fails in the Quadruped domain. Similarly, all the other baseline methods cannot collect consistent high-quality datasets for all domains. In contrast, the dataset collected by CUDC demonstrates a higher quality with an expanded feature space, as the offline agent's learning performances at 500K steps are enhanced in all 12 downstream tasks across the 3 challenging domains, as highlighted in Table \ref{table:main} of Appendix \ref{sec:full_main}. Specifically, CUDC outperforms the competence-based method (APS) in the Walker domain by 6\%, outperforms the data-based method (APT) in the Quadruped domain by 51\%, and outperforms the knowledge-based method (ICM) in the Jaco Arm domain by 10\%. In terms of efficiency, Figure~\ref{fig:mian_two_curves} shows significant improvements of CUDC on 3 downstream tasks of the Quadruped domain, indicating improved computational efficiency. In the easiest domain of Walker, CUDC helps the offline agent to converge faster in 3 downstream tasks. However, the computational efficiency in the Jaco Arm domain is unsatisfactory. This could be due to too much complexity in this most challenging environment, increasing the difficulty of reachability analysis. A visualization of the quality for the collected datasets is provided in Appendix \ref{sec:full_main}, where our proposed method has collected higher-quality dataset with increasingly more rewarding states being visited. For the sample efficiency, the offline RL agent can perform well with significantly less data collected by the proposed CUDC as discussed in Appendix~\ref{sec:sample_efficiency}. Additional results are presented in Appendix~\ref{sec:full_main} and consistent results are obtained by evaluating with another offline RL algorithm of CQL \citep{kumar2020conservative} in Appendix~\ref{sec:cql_results}.

	\paragraph{Effects of Adapting the $k$-Step} We empirically show that adapting the temporal distance to explore more distant future states can enhance the feature representation, and thereby improve the data collection process. By comparing the results in Figure~\ref{fig:variants},  $\text{CUDC}_\text{vary}^\text{ICM}$ has outperformed ICM significantly, with on average 1.25 $\times$  computational efficiency at 100K step and 1.16 $\times$ offline learning performance at 500K step. Similarly, $\text{CUDC}_\text{vary}^\text{APT}$ obtains respectively 1.12 $\times$ and 1.04 $\times$ scores at 100K and 500K steps across 4 downstream tasks, compared with APT. Note that the standard deviation increases slightly, which may be due to the introduced complexity of considering more distant future states in improving the learned representation.  Thus, it is important to find an adaptive way to smooth this process, such as by incorporating the other proposed components coherently.
	
		\begin{figure*}[t]
		\centering
		\vspace{-1mm}
		\includegraphics[width=1\textwidth]{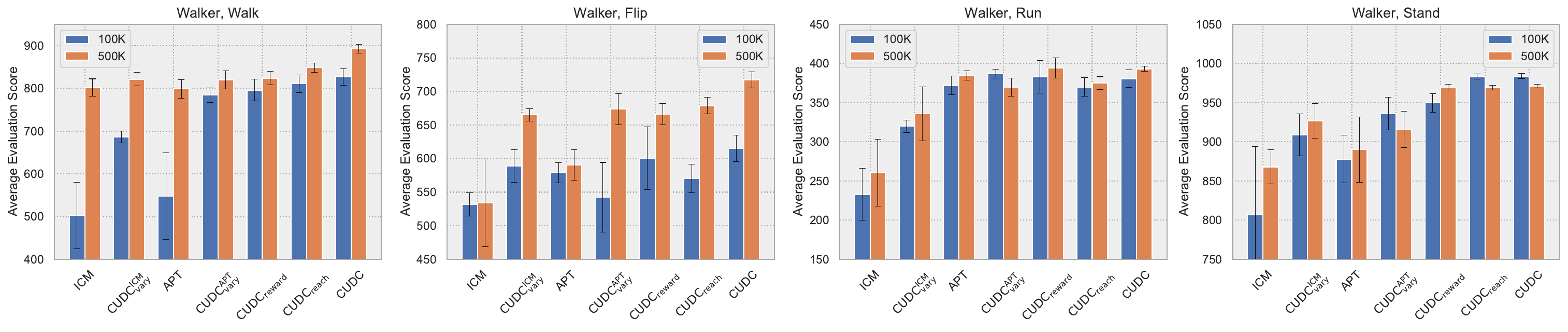} 
		\vspace{-7mm}
		\caption{\small The performance score evaluated at 100K and 500K steps in 4 downstream tasks of Walker. All four versions of CUDC perform better than ICM and APT.}
		\label{fig:variants}
	\end{figure*}
 	    \begin{figure*}[t]
		\centering
            \vspace{-2mm}
		\includegraphics[width=0.95\textwidth]{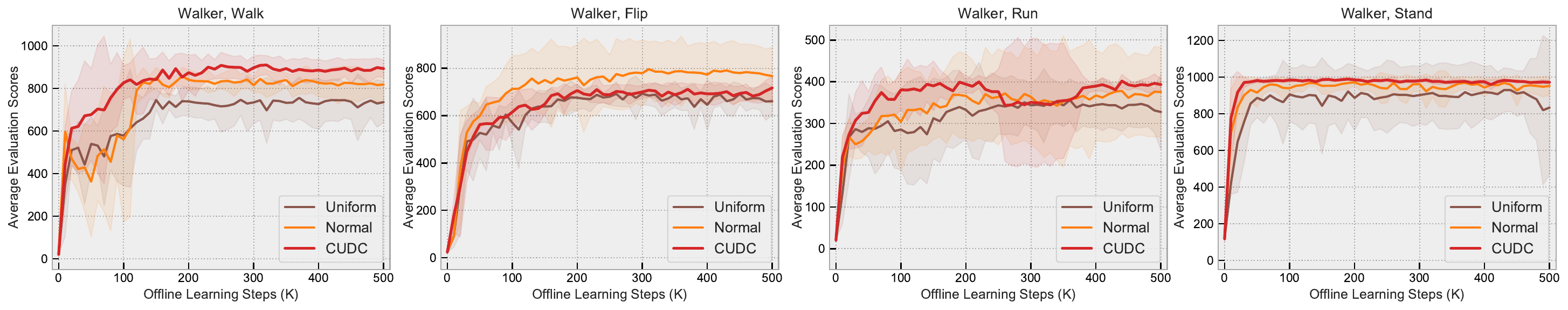} 
		\vspace{-3mm}
		\caption{\small Learning curves of the offline RL agent on 4 downstream tasks of Walker. The $k$-step adaptation proposed in CUDC outperforms the other two sampling methods in 3 out of 4 downstream tasks.}
		\vspace{-3mm}
		\label{fig:ablation_main_curve)}
	\end{figure*}
	\normalsize

    \paragraph{Effectiveness of the Other Proposed Components} We additionally integrated mixed intrinsic reward into $\text{CUDC}_\text{vary}^\text{APT}$ as $\text{CUDC}_\text{reward}$, resulting in further improvements in learning efficiency at 100K steps by 3.3\% and capability at 500K steps by 3.0\% for the offline RL agent, as presented in Figure~\ref{fig:variants}. However, due to the mixed intrinsic reward's nature of promoting uniform exploration in the state-action space and focusing on under-learned states, the performance at 100K steps became unstable with a 67\% increase in standard deviation. Thus, we leveraged the proposed reachability module to function as the agent's internal belief and facilitate the data collection process. Comparing $\text{CUDC}_\text{reach}$ and $\text{CUDC}_\text{reward}$, the dataset collected by $\text{CUDC}_\text{reach}$ reduced standard deviation by 48\% and 25\% at 100K and 500K steps, respectively, stabilizing offline learning. However, its performance scores decreased slightly in two tasks. The full model, compared to $\text{CUDC}_\text{reach}$, further regularizes the critic-actor update with the curiosity weight to focus more on under-learned tuples, resulting in a 3.2\% and 4.0\% improvement in learning efficiency and capability, respectively, with the minimum standard deviation at 500K steps. To further investigate the effectiveness, we carry out more experiments by respectively removing each proposed component from the full model in Appendix~\ref{sec:ablation_remove}.
	It can be concluded that varying the temporal distance is the most crucial factor in collecting a useful dataset with an expanded feature space, while the other components work coherently to yield further improvement. 
 

	\paragraph{Adjusting the $k$-Step in Different Ways} One may be curious about how adjusting the temporal distance $k$ in the reachability module affects the feature representation. To investigate this, we conducted an ablation study in the Walker domain by sampling $k$ uniformly (Uniform) from 3 to 6 and normally (Normal) with an increasing mean. The results in Figure~\ref{fig:ablation_main_curve)} show that Uniform performs the worst in all 4 tasks as it cannot adapt the temporal distance in a way that enhances representation learning. At 500K steps, it only achieves 85\% overall learning capability with a 300\% increase in standard deviation, compared to CUDC. Normal to some extent adapts $k$ through an increasing mean, and it even outperforms CUDC in the Flip task. However, its overall performance is still 4.5\% weaker than CUDC, and its standard deviation is 128\% higher than CUDC, indicating an instability issue. Overall, the curious adaptation method proposed in CUDC is the best, and there is potential to investigate more adaptive ways in the future.

   \paragraph{Limitations and Broader Impacts}
   Despite demonstrating strong empirical performance, CUDC is not without its limitations. Like other unsupervised methods, its scalability to complex environments may be limited, and in safety-critical applications where expert data is crucial, relying solely on unsupervised approaches can pose risks. From an ethical standpoint, the application of CUDC in real-world scenarios, such as robotics, AI video games, or social media platforms, raises concerns. The process of diverse data collection without proper supervision or restrictions can give rise to potential safety and privacy issues. It is important to address these ethical considerations to ensure the responsible and safe implementation of the CUDC method in practical applications.
   
	\section{Conclusion}
	We propose CUDC, a curiosity-driven unsupervised data collection method for multi-task offline RL. It dynamically expands the feature space to improve dataset quality. CUDC includes a reachability module that estimates the probability of a $k$-step future state being reachable from the current state. By adaptively allowing the agent to explore more distant future states, CUDC can enhance feature representation. Empirically, our method outperforms existing unsupervised benchmarks in terms of computational efficiency, sample efficiency, and learning performance. Our work provides valuable insights into effective data collection methods for future research.

\section{Acknowledgments}
This research is supported by Alibaba Group and Alibaba-NTU Singapore Joint Research Institute (JRI), Nanyang Technological University, Singapore. H.Qian thanks the support from CFAR.

\bibliography{aaai24}
\normalsize	
\newpage
\appendix
\setcounter{figure}{0}
\onecolumn
\section{Implementation and Hyperparameter Settings}
\label{appendix:hyper}
\subsection{Compute and Overall Settings}
We conduct experiments on four cloud severs and one physical server with the following configurations.
\begin{itemize}
	\item Operation System: Ubuntu 18.04
	\item Memory: 32GiB / 32GiB / 32GiB / 32GiB / 128GiB
	\item CPU: Intel Core Processor (Skylake) / Intel Core Processor (Skylake) / Intel Core Processor (Skylake) / Intel Core Processor (Skylake) / Intel(R) Xeon(R) CPU E5-2620 v2 @ 2.10GHz
	\item vCPU: 8 / 8 / 16 / 16 / 24
	\item GPU: 2 $\times$ NVIDIA Tesla P100 16GB / 2 $\times$ NVIDIA Tesla P100 16GB / 1 $\times$ NVIDIA Tesla V100S PCIE 32GB / 1 $\times$ NVIDIA Tesla V100S PCIE 32GB / 2 $\times$ NVIDIA GeForce RTX 3090 24GB
\end{itemize}
Our proposed CUDC is implemented using PyTorch \cite{paszke2019pytorch} based on ExORL benchmark \cite{yarats2022don}. \footnote{The ExORL benchmark is available at \url{https://github.com/denisyarats/exorl}.} The majority of ExORL is licensed under the MIT license, while portions of the project (DeepMind \cite{tassa2018deepmind}) is licensed under the Apache 2.0 license.
It approximately takes 8 hours to collect a 1M sized dataset while 4 hours to perform one downstream task offline learning. 
The implementation details are as follows.
\subsection{Implementation Details}
\subsubsection{CUDC}
	In the reachability module of CUDC shown in Figure~1~(right), the state encoder $\phi_s$ is a 1-layer MLP with the ReLU activation. Subsequently, the output is passed to a single normalized fully-connected layer by LayerNorm \citep{lei2016layer} with the tanh nonlinearity applied at the end. The target encoder is momentum updated from the state encoder to obtain $\bar{z}$. The action encoder $\phi_a$ is a 3-layer MLP with the ReLU activation. For the forward network $f_s$ and the backward networks $f_a$ and $f_k$, they are 2-layer MLP with the ReLU activation. The projection network $h$ is a 2-layer MLP with hidden size of 128 and output size of 64, followed by LayerNorm. $\bar{h}$ is also momentum updated from $h$.
	
	\subsubsection{$\text{CUDC}_\text{vary}^\text{ICM}$ and $\text{CUDC}_\text{vary}^\text{APT}$}
	$\text{CUDC}_\text{vary}^\text{ICM}$ and $\text{CUDC}_\text{vary}^\text{APT}$ are two variants of the proposed CUDC. Given the implementation of ICM and APT in ExORL \citep{yarats2022don}, we added the forward network $f_s$ and backward network $f_k$ in order to let these models learn more feature representations with the dynamics information. Therefore, the temporal distance of $k$-step can be adapted in a similar way as in CUDC. $k$ is increased by 1 if the proportion of tuples with low intrinsic rewards is greater than a threshold, i.e.  $\frac{\sum_i^n \mathds{1}_\mathrm{r_i < C_r}}{n} > C_k$. In this way, we can investigate the effects of adapting environment step on two unsupervised baseline methods.
	
	\subsubsection{ $\text{CUDC}_\text{reward}$}
	$\text{CUDC}_\text{reward}$ is a variant of the proposed CUDC. Given the implementation of $\text{CUDC}_\text{vary}^\text{APT}$ in the previous subsection, we extended the state entropy to state-action entropy and meanwhile add a prediction error of the $k$-step future state as formulated in Equation 3. Therefore, agent can be encouraged to explore more diverse state-action space while focusing on the under states with high prediction errors.  In this way, we can investigate the effects of the proposed mixed intrinsic reward.
	
	\subsubsection{ $\text{CUDC}_\text{reach}$}
	$\text{CUDC}_\text{reach}$ is a variant of the proposed CUDC. Given the implementation of $\text{CUDC}_\text{reward}$ in the previous subsection, we incorporated the full reachability module with the reachable network. As a result, the adaptive update of the environment step can be facilitated by the curiosity weight $w$ outputted by the reachability module and the enhanced representation learning can be carried out in a self-supervised manner. Compared with the full CUDC model, $\text{CUDC}_\text{reach}$ has disabled the regularization of the critic-actor update. In this way, we can investigate the effects of the proposed reachability module.

	\subsection{Hyperparameter Setting}
	
	\subsubsection{Data Collection}
	We provide a full set of common hyperparameters used in baselines as well as CUDC in Table~\ref{tab:common_hyper}, which closely follows the same settings from ExORL \citep{yarats2022don} and URLB \citep{laskin2021urlb}.
	
	\begin{table}[h!]
		\centering
		\caption{Common hyperparameter setting for the unsupervised data collection methods}
            \vskip 0.1in
		\label{tab:common_hyper}
		\begin{tabular}{lr}
			\toprule
			Hyperparameter           & Value     \\* \midrule
			Observation type         & states    \\
			Replay buffer Size       & $10^6$    \\
			Action repetitions       & 1         \\
			Seed frames              & 4000      \\
			Batch size               & 1024      \\
			Discount factor                & 0.99      \\
			Optimizer                & Adam      \\
			Learning rate            & $10^{-4}$ \\
			Non-linearity            & ReLU      \\
			Agent update frequency   & 2         \\
			Critic target EMA rate   & 0.01      \\
			Hidden dimension         & 1024      \\
			Exploration stddev clip  & 0.3       \\
			Exploration stddev value & 0.2       \\* \bottomrule
		\end{tabular}
	\end{table}
	
	For the other hyperparameter used in CUDC, they are listed in Table~\ref{tab:CUDC_hyper}. The random seeds are consistent with ExORL.
	
	\begin{table}[h!]
		\centering
		\caption{Hyperparameter setting for the proposed CUDC}
  \vskip 0.1in
		\label{tab:CUDC_hyper}
		\begin{tabular}{ll}
			\toprule
			Hyperparameter                               & Value                                 \\ \midrule
			$k$-step range                       & 3, 4, 5, 6                         \\
			State representation dimension               & 512                                \\
			Actor representation dimension               & 64                                 \\
			MLP hidden dimension for action encoder      & 64 for action encoder              \\
			MLP hidden dimension                         & 128 for projection                 \\
			Projection dimension                         & 64                                 \\
			Regularization clip                          & {[}0.2, 2{]} for Walker            \\
			& {[}0.2, 1{]} for Quadruped         \\
			& {[}0.9, 1{]} for Jaco Arm          \\
			Intrinsic reward weights $(\alpha,   \beta)$ & $(10^{-3}, 10^{-2})$ for Walker    \\
			& $(10^{-4}, 10^{-2})$ for Quadruped \\
			& $(100, 0)$ for Jaco Arm            \\
			K-NN                                         & 12                                 \\
			Threshold $C_w$                              & 0.02 for Walker and Quadruped      \\
			& 0.01 for Jaco Arm                  \\
			Threshold $C_k$   & 0.5  \\  \bottomrule
		\end{tabular}
	\end{table}
	
	It should be noted that Explore2Offline \citep{lambert2022challenges} is a concurrent work for data collection, but its source code is not available at the moment.
	
	\subsubsection{Offline RL}
	For the offline RL agent, we follow the findings reported in ExORL that even the vanilla off-policy RL algorithm of TD3 \citep{fujimoto2018addressing} can outperform the carefully designed offline RL algorithms when the collected dataset is of high quality. Thus, we implement an offline RL of TD3 to evaluate the quality of the collected dataset. In addtion, another offline RL algorithm of CQL \citep{kumar2020conservative} is also included to further verify the higher-quality of the dataset collected by CUDC. We follow the same setting with default random seeds as ExORL to carry out the experiments. The detailed implementation and hyperparamter settings for both TD3 and CQL can be found in ExORL.
 \\
 \newpage
  \clearpage
 \newpage
	
	\section{Proof of Lemma 3.1}
 \label{sec:proof}
	\begin{proof}
		It has been shown in APT \citep{liu2021behavior} that the particle-based entropy estimator of the state representation $z$ can be derived as
		\begin{equation}
		\mathcal{H}(z) = -\frac{1}{n}\sum_{i=1}^n \log \frac{K}{n v_i^K} +b(K) \propto  \sum_{i=1}^{n} \log v_i^K
		\end{equation}
		where $b(K)$ represents a bias correlation and $v_i^K$ indicates the volume of the hypersphere with a radius of $||z_i-z_i^{\text{K-NN}}||$ between the $z_i$ and its $K$-th nearest neighbor $z_i^{\text{K-NN}}$.
		
		By substituting $v_i^K=\frac{||z_i-z_i^{\text{K-NN}}|| \pi^{n_z/2}}{\Gamma(n_z/2+1)}$ where $\Gamma$ is a gamma function and $n_z$ is the dimension of $z$, we can obtain
		\begin{equation}
		\label{entorpy_z}
		\mathcal{H}(z) \propto  \sum_{i=1}^{n} \log||z_i - z_i^{\text{K-NN}}||_2.
		\end{equation}
		Let $u=(z_s, z_a)$ represent the state-action representation and we can further substitute $z=u$ into Equation (\ref{entorpy_z}) to obtain
		\begin{equation}
		\mathcal{H}(u) \propto  \sum_{i=1}^{n} \log||u_i - u_i^{\text{K-NN}}||_2.
		\end{equation}

	\end{proof}

	\newpage
	\section{Additional Results and Discussions}
	\subsection{Full Results of Main Experiments}
	\label{sec:full_main}
	Table \label{table:main}, Figure \ref{fig:main_bar_full} and Figure \ref{fig:main_curve_full)} show the full results on the 12 downstream tasks across 3 domains. Figure~\ref{fig:overall_main_bar} summarizes the overall performances in 3 domains. They demonstrate that our proposed methods significantly outperform the baseline methods in 3 downstream tasks of Walker and 3 downstream tasks of Quadruped. However, in the hardest domain of Jaco Arm, the computational efficiency becomes poor in 3 downstream tasks although its learning performance at 500K step catches up with the best baseline method of ICM. We explain this by the fact that our method can introduce more complexity and challenges for the agent whenever adapting the temporal distance $k$ to learn a better representation. However, as Jaco Arm is indeed the most challenging domain, the introduced complexity cannot be fully coped with in this environment. Thus, the poor performance in computational efficiency can occur, and it is interesting to further study on how to cope with hard domain environments in future works.
	    \begin{table*}[h]
    \vskip -0.15in
		\caption{Main results of the offline RL agent on 12 downstream tasks across 3 main domains. The proposed CUDC collects a more useful dataset such that an offline RL agent can improve computational efficiency (100K) in 9 out of 12 downstream tasks and achieve better learning performance (500K) in all 12 downstream tasks. } 
		\label{table:main} 
		\begin{center}
            \resizebox{0.9\columnwidth}{!}{%
            \begin{tabular}{lllllll}
			\toprule
			100K Step Score              & Random      & APS         & ProtoRL              & ICM         & APT         & CUDC                \\* \midrule
			Walker, Walk                & 190$\pm$153 & 652$\pm$227         & 532$\pm$332 & 503$\pm$258         & 548$\pm$338 & \textbf{827$\pm$64}  \\
			Walker, Flip                & 175$\pm$136 & 590$\pm$77          & 610$\pm$95  & 530$\pm$59          & 579$\pm$51  & \textbf{615$\pm$66}  \\
			Walker, Run                 & 53$\pm$25   & 368$\pm$77          & 332$\pm$81  & 233$\pm$111         & 372$\pm$40  & \textbf{381$\pm$37}  \\
			Walker, Stand               & 401$\pm$295 & 923$\pm$76          & 831$\pm$318 & 797$\pm$312         & 878$\pm$101 & \textbf{984$\pm$11}  \\
			Quadruped, Walk              & 135$\pm$101 & 206$\pm$25 & 153$\pm$119 & \textbf{338$\pm$211} & 248$\pm$22 & \textbf{338$\pm$147} \\
			Quadruped, Run              & 145$\pm$96  & 183$\pm$17          & 121$\pm$31  & 210$\pm$63          & 249$\pm$24  & \textbf{256$\pm$133} \\
			Quadruped, Stand            & 271$\pm$105 & 334$\pm$148         & 193$\pm$78  & 585$\pm$301         & 524$\pm$153 & \textbf{618$\pm$311} \\
			Quadruped, Jump             & 223$\pm$60  & 237$\pm$83          & 150$\pm$89  & 466$\pm$208         & 391$\pm$127 & \textbf{483$\pm$222} \\
			Jaco Arm, Reach Top Left    & 5$\pm$4     & 63$\pm$40           & 68$\pm$49   & \textbf{88$\pm$78}  & 42$\pm$77   & 54$\pm$60            \\
			Jaco Arm, Reach Top Right   & 49$\pm$37   & \textbf{119$\pm$64} & 76$\pm$39   & 99$\pm$73           & 51$\pm$67   & 32$\pm$59            \\
			Jaco Arm, Reach Bottom Left & 35$\pm$31   & 87$\pm$74           & 76$\pm$75   & \textbf{101$\pm$46} & 30$\pm$43   & 74$\pm$89            \\
			Jaco Arm, Reach Bottom Right & 49$\pm$46   & 100$\pm$75 & 85$\pm$66   & 113$\pm$89           & 92$\pm$42  & \textbf{121$\pm$79}  \\ \midrule
			500K Step Score             & Random      & APS                 & ProtoRL     & ICM                 & APT         & CUDC                 \\ \midrule
			Walker, Walk                & 198$\pm$266 & 845$\pm$38          & 826$\pm$67  & 802$\pm$67          & 799$\pm$73  & \textbf{893$\pm$34}  \\
			Walker, Flip                & 303$\pm$195 & 561$\pm$121         & 645$\pm$150 & 534$\pm$218         & 591$\pm$77  & \textbf{717$\pm$41}  \\
			Walker, Run                 & 62$\pm$24   & 369$\pm$33          & 386$\pm$38  & 261$\pm$142         & 384$\pm$20  & \textbf{393$\pm$11}  \\
			Walker, Stand               & 519$\pm$312 & 949$\pm$24          & 954$\pm$17  & 868$\pm$73          & 890$\pm$139 & \textbf{971$\pm$7}   \\
			Quadruped, Walk             & 77$\pm$53   & 169$\pm$38          & 177$\pm$181 & 231$\pm$107         & 363$\pm$141 & \textbf{425$\pm$76}  \\
			Quadruped, Run              & 102$\pm$52  & 179$\pm$52          & 77$\pm$36   & 165$\pm$89          & 198$\pm$57  & \textbf{349$\pm$80}  \\
			Quadruped, Stand            & 162$\pm$96  & 335$\pm$80          & 127$\pm$111 & 343$\pm$133         & 464$\pm$144 & \textbf{637$\pm$188} \\
			Quadruped, Jump             & 152$\pm$74  & 261$\pm$71          & 99$\pm$51   & 242$\pm$108         & 329$\pm$85  & \textbf{574$\pm$74}  \\
			Jaco Arm, Reach Top Left    & 59$\pm$75   & 129$\pm$26          & 138$\pm$52  & 166$\pm$54          & 41$\pm$29   & \textbf{212$\pm$11}  \\
			Jaco Arm, Reach Top Right    & 81$\pm$54   & 152$\pm$82 & 166$\pm$19  & 195$\pm$36           & 95$\pm$32  & \textbf{214$\pm$17}  \\
			Jaco Arm, Reach Bottom Left  & 91$\pm$68   & 103$\pm$74 & 100$\pm$68  & \textbf{216$\pm$16}  & 101$\pm$56 & \textbf{218$\pm$7}   \\
			Jaco Arm, Reach Bottom Right & 107$\pm$56  & 197$\pm$33 & 149$\pm$69  & \textbf{229$\pm$8}   & 131$\pm$33 & \textbf{229$\pm$6}    \\* \bottomrule
		\end{tabular}
            }
		\end{center}
            \vskip -0.2in
		
	\end{table*}
	\normalsize
 
	\begin{figure}[H]
		\centering
        \vskip -0.1in
		\includegraphics[width=0.95\textwidth]{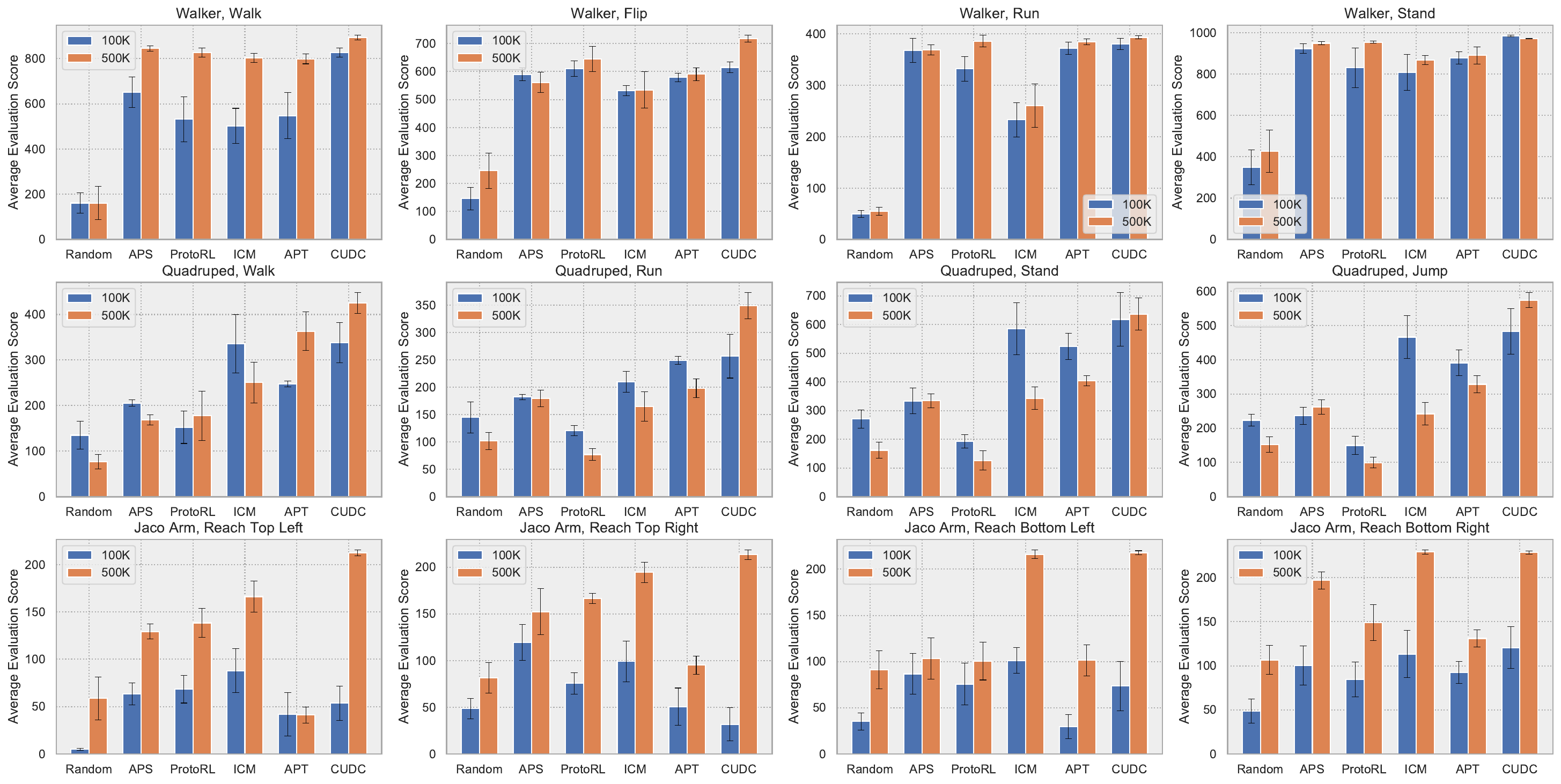}
		\caption{Full results of the average performance score evaluated at 100K and 500K steps for TD3 on each downstream task. CUDC significantly improves the computational efficiency in 9 out 12 tasks at 100K step and learning performances in all 12 tasks at 500K step.}
		\label{fig:main_bar_full}
        \vskip -0.3in
	\end{figure}
	\begin{figure}[H]
		\centering
		\includegraphics[width=0.95\textwidth]{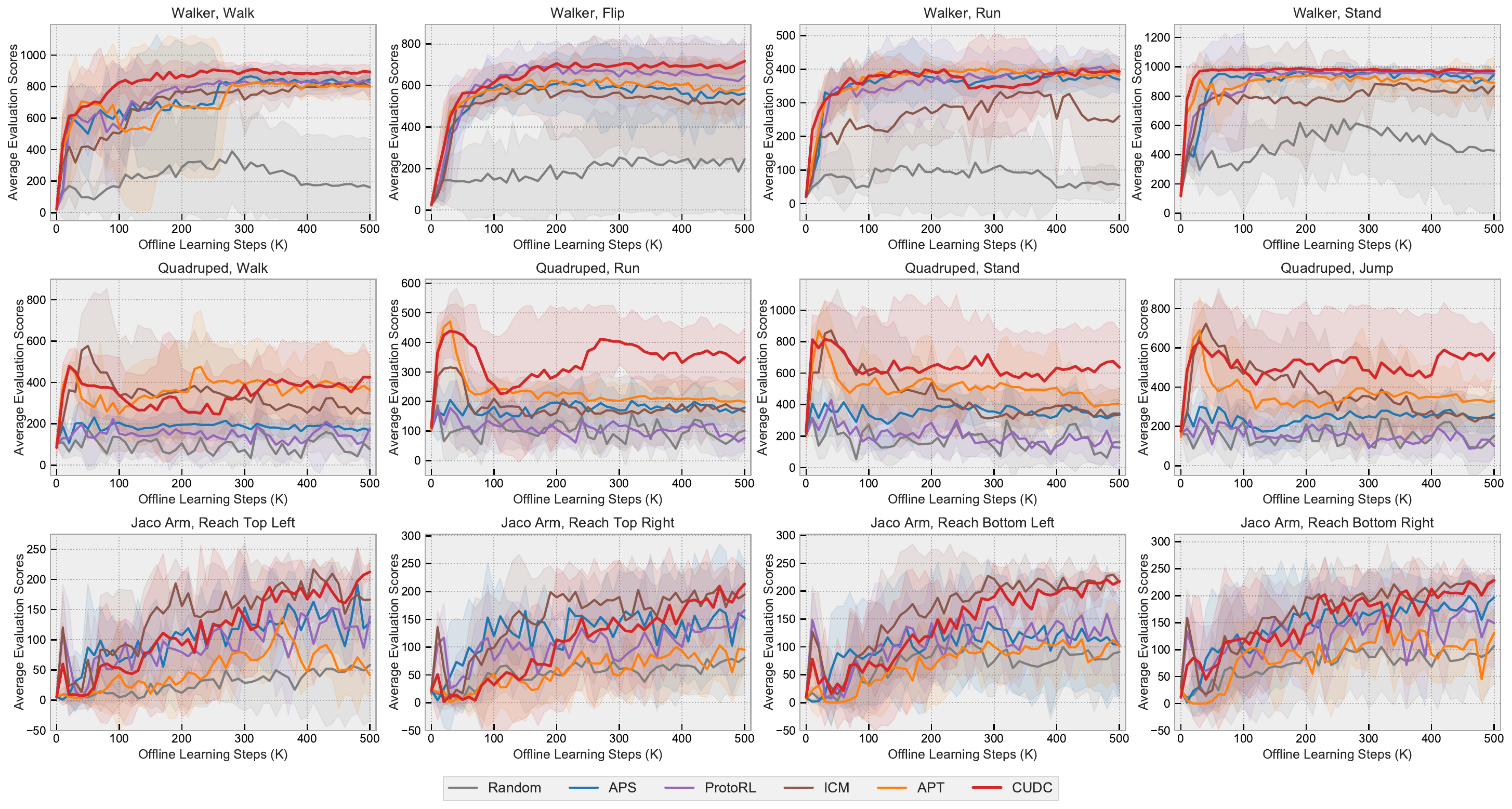} 
  \vskip -0.1in
		\caption{Learning curves of the offline RL agent on full 12 downstream tasks with the task-agnostic dataset collected by different methods. The proposed CUDC demonstrates superior capability of improving the computational efficiency and learning performances of the offline RL agent in the domains of Walker and Quadruped.}
		\label{fig:main_curve_full)}
 \vskip -0.2in
	\end{figure}
	
	\begin{figure}[H]
		\centering
		\includegraphics[width=1\textwidth]{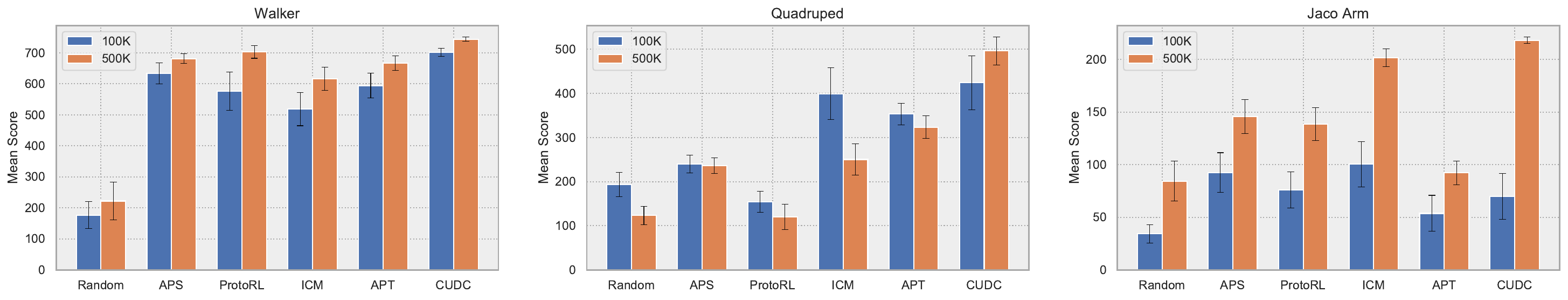} 
		\caption{The overall offline learning performance across 3 domains. CUDC consistently leads to significant improvement in offline agent's performance at 500K step in all 3 domains on average. }
		\label{fig:overall_main_bar}
	\end{figure}
	\normalsize
	
	In Figure~\ref{fig:kstep_change_walker}, we plot how our proposed method adapts the $k$-step during data collection in the Walker domain. It can be observed that agents take around 450K step to learn and adjust the temporal distance from $k=3$ to $k=4$. Then, it takes around just 200K steps to increase from $k=4$ to $k=5$. Finally, it takes more than 400K steps to reach $k=6$. This increasing behavior implies that at the early phase of training, we should not inject a too challenging $k$ value. Once the agent has learned enough dynamics information, they are able to learn quickly on more challenging reachability analysis. However, after a certain stage ($k=5$ in this domain), the learning becomes too difficult for the agent. Therefore, it is interesting to further adapt this over-difficult knowledge in future works as well.

	\begin{figure}[h]
		\centering
		\includegraphics[width=0.4\textwidth]{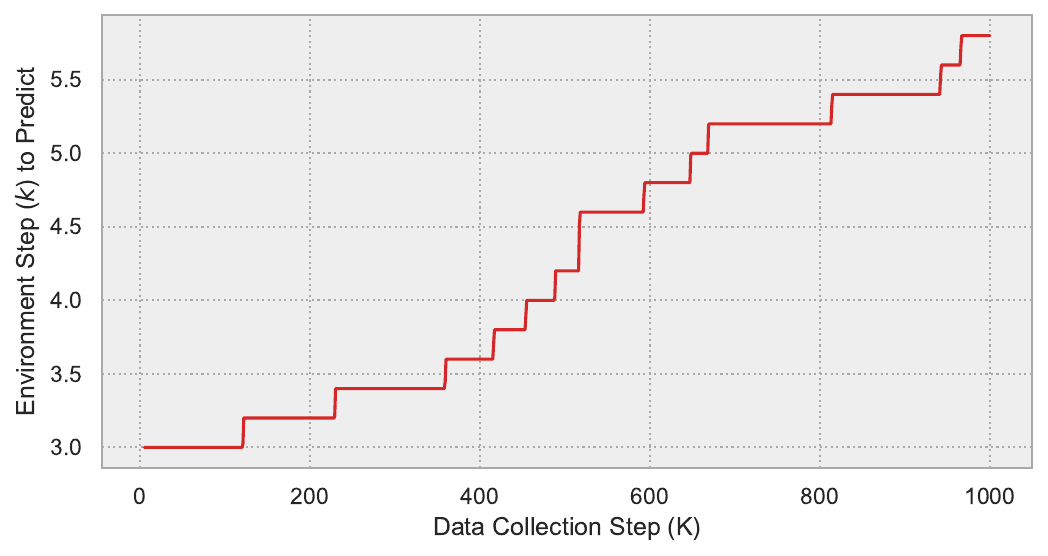}
		\vspace{-2mm}
		\caption{\small The adaptive increase of how many steps into the future that the dynamics model should predict. It is averaged by 5 random runs in the Walker domain, where the k-step is limited to change from 3 to 6.}
		\label{fig:kstep_change_walker}
	\end{figure}

	In Figure~\ref{fig:loss_visualiztion}, we visualize how the loss of dynamics model (forward and inverse networks) and the contrastive loss of the reachable networks change w.r.t. training steps. Both losses decrease and converge during learning. In particular, they only increase when the step $k$ is adapted to increase. After that, both losses decrease quickly. The agent can well predict each sampled $k$-step future state being most reachable from its own current state rather than those from other transition tuples. Even without inputting the sequence of intermediate actions or states, the agent can predict the $k$-step future state accurately. During the model training, the feature representations are enhanced as well. Therefore, it can be validated that the learned representation contains more semantic latents with dynamics information, with the self-supervised learning of agent's internal belief.

	\begin{figure}[h]
		\centering
		\includegraphics[width=0.8\textwidth]{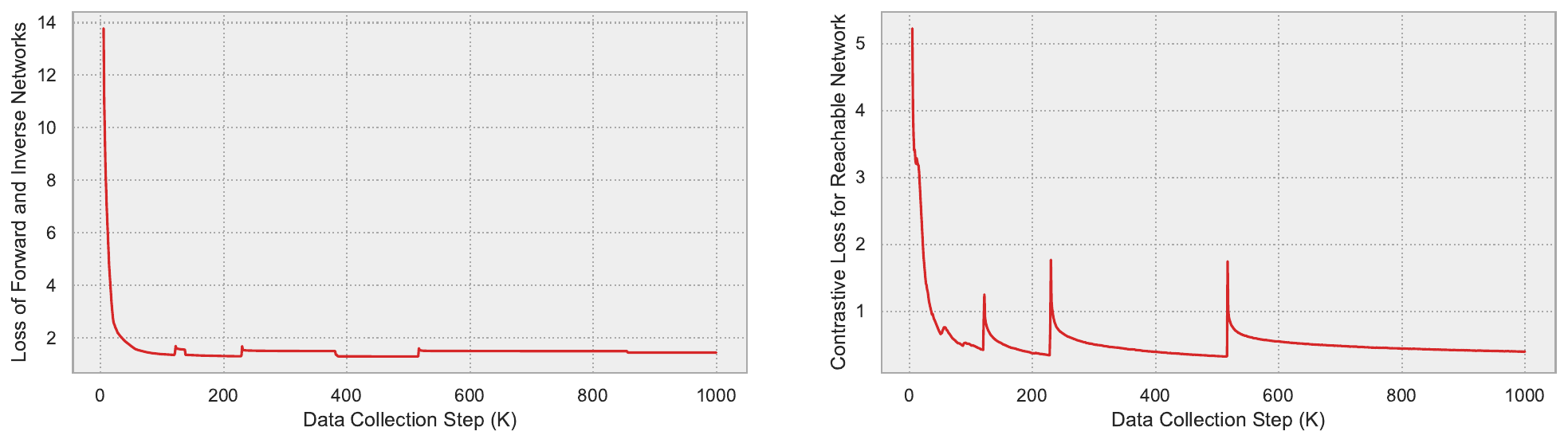} 
		\vspace{-2mm}
		\caption{The visualization of the dynamics model loss as well as reachability loss (contrastive loss) during learning in the Walker domain.}
		\label{fig:loss_visualiztion}
	\end{figure}
	
    In addition to the offline multi-task learning performance, we compare the quality of the collected datasets by plotting the normalized density of the true reward for the downstream task of Stand in Walker environment. The visualization is shown in Figure~\ref{fig:distribution_visualiztion}. It can be seen that the dataset collected by our proposed CUDC is of higher quality, with larger density for high rewards and a larger proportion at the low reward part. Specifically, the mean trajectory reward of CUDC is 0.312, which is 69\% higher than APS, 21\% higher than ProtoRL, 16\% higher than ICM, and 10\% higher than APT. Moreover, the 75\% quartile of the trajectory reward for CUDC is 0.498, which is 143\% larger than APS, 42\% larger than ProtoRL, 19\% larger than ICM, and 16\% larger than APT. Therefore, it indicates the effectiveness of our proposed method to collect higher-quality dataset, where increasingly more rewarding states have been visited.
    \begin{figure}[h]
		\centering
		\includegraphics[width=1\textwidth]{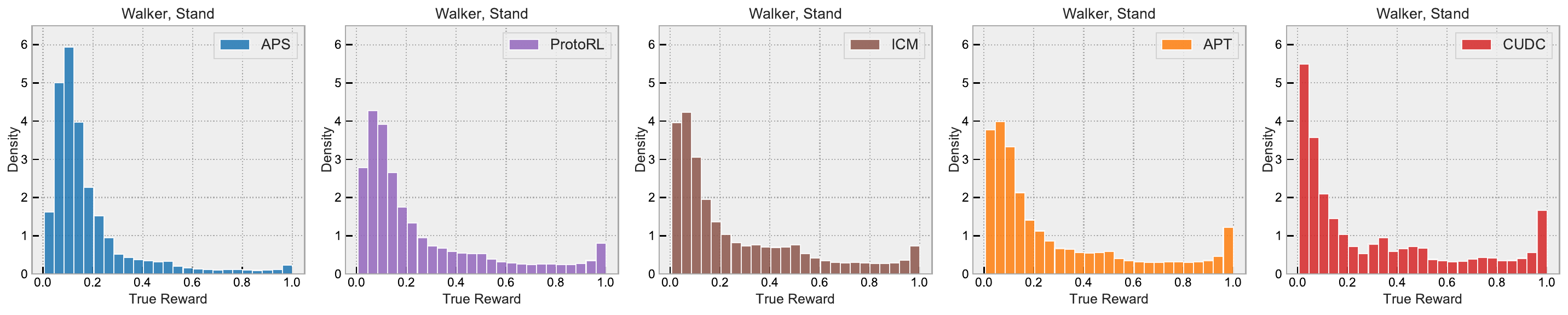} 
		\caption{The visualization of the true reward density on Walker, Stand task for the task-agnostic dataset collected by the proposed method and baseline models.}
		\label{fig:distribution_visualiztion}
	\end{figure}

\subsection{Evaluation on Sample Efficiency}
In this subsection, we demonstrate the sample efficiency of our proposed method from three perspectives: 1) the offline RL performance with different sized dataset, 2) proportion of useful data in the fixed dataset for offline learning, and 3) reusability of the task-agnostic dataset across many downstream tasks.

Firstly, we evaluate the offline RL performance using TD3 with 500K gradient step learning from respectively 100K, 200K, 500K, and 1000K sized dataset. On the one hand, it can be observed from Figure~\ref{fig:sample_efficiency} that CUDC has clearly enabled the offline RL agent to learn better in all four downstream tasks of Walker at almost all sizes. For example, with a fixed 100K dataset, the offline RL agent can perform 46.4\% higher in Walk, 22.6\% higher in Flip, 42.4\% higher in Run, and 1.3\% higher in Stand, by learning from the dataset collected by the proposed method against the best baseline method. On the other hand, the offline RL agent requires significantly less data collected by the proposed CUDC to achieve a certain level of performance, compared to the other baseline methods. As shown in the first sub-figure of Figure~\ref{fig:sample_efficiency}, to achieve the same performance of learning from 100K CUDC-collected dataset in the Walk task, the offline RL agent needs approximate 310K ICM-collected dataset, 180K APT-collected dataset, 340K APS-collected dataset, and 160K ProtoRL-collected dataset.

\label{sec:sample_efficiency}
	\begin{figure}[h]
		\centering
		\includegraphics[width=1\textwidth]{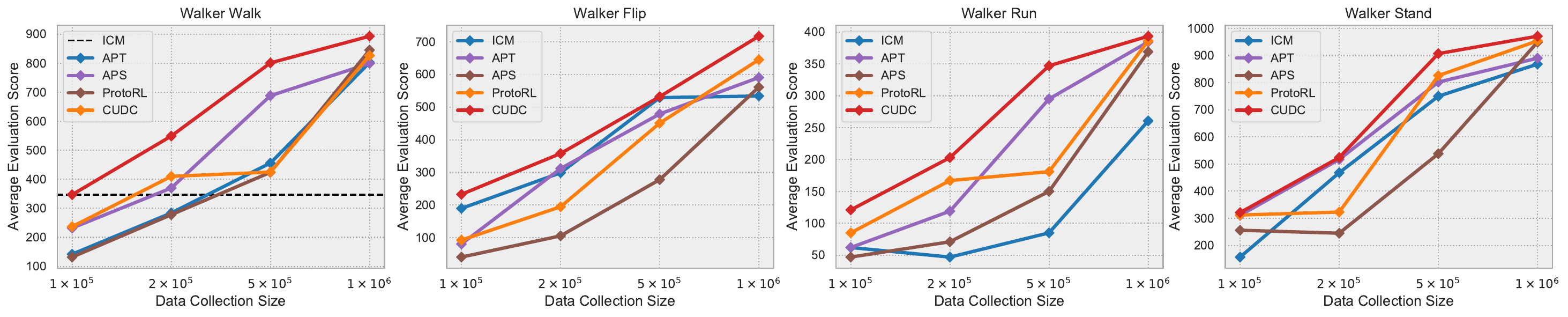} 
		\caption{Learning scores with different dataset sizes for the TD3 agent on 4 downstream tasks of Walker. Compared to the baseline methods, the proposed CUDC demonstrates improved sample efficiency by collecting different amounts of data that consistently lead to the highest performance score at 500K learning steps for the offline RL agent.}
		\label{fig:sample_efficiency}
	\end{figure}	

Secondly, as shown in Figure 2 despite the same amount of data being sampled at each gradient step, our proposed method outperforms the baseline methods at both 100K and 500K learning steps in most downstream tasks. Thus, we observe that not all collected samples are useful for the offline RL agent and our method is sample efficient in the sense that it collects higher-quality dataset with a larger proportion of useful data points in the fixed dataset for the downstream tasks. We have visualized this in Figure 10 by plotting the true reward distribution of the collected data points on a downstream task.

Lastly, our main motivation in this paper is to re-use the collected dataset across many downstream tasks, which avoids the expensive data collection for each individual downstream task and thereby improves the sample efficiency. According to the main results in Figure 2 and Table \ref{table:main}, our method has outperformed the baseline methods in the sense that the offline RL agent can perform multiple downstream tasks well with a single task-agnostic dataset.

	\subsection{Evaluation by Another Offline RL Algorithm of CQL}
	\label{sec:cql_results}
	In this work, we consider the problem setting for offline RL into three main steps: data collection, reward relabeling, and downstream offline learning. Our work focuses only on the first step of collecting high-quality dataset and thereby is agnostic to the offline RL algorithms. For the main experiments, we chose TD3 \citep{fujimoto2018addressing} to evaluate the multi-task downstream learning, since it was concluded in ExORL that the vanilla TD3 algorithm can effectively learn offline and even outperform carefully designed offline RL algorithms by improving the dataset quality. 
	
	To demonstrate that the dataset collected by our proposed method is of high quality than the other methods, we conduct additional experiments using another offline RL algorithm of CQL~\citep{kumar2020conservative} that regularizes the Q-values during training. The results are shown in Figure~\ref{fig:cql_curve} and our proposed CUDC has also demonstrated improved computational efficiency at 100K learning steps and learning capability at 500K learning steps in all 4 downstream tasks against baseline methods.
	By comparing the performance score at 100K and 500K steps of CUDC against the best baseline method, CUDC has achieved on average 18.2\% and 12.3\% improvement at respectively 100K and 500K CQL agent learning steps, across 4 downstream tasks. Specifically, CUDC is 21\% higher than the best baseline of APT at 100K steps on Run task while its learning performance is 15\% higher than the best baseline of APS at 500K steps.
	
	\begin{figure}[H]
		\centering
		\includegraphics[width=1\textwidth]{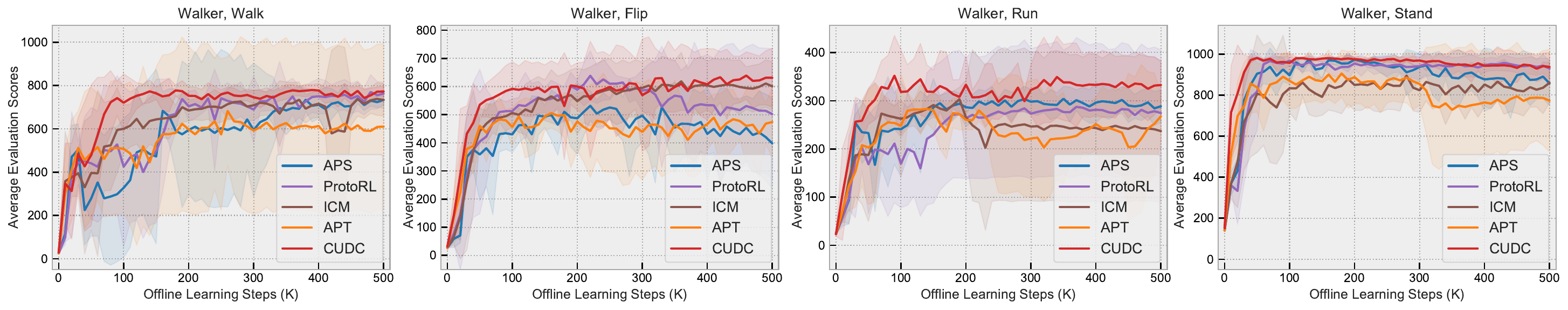} 
		\includegraphics[width=1\textwidth]{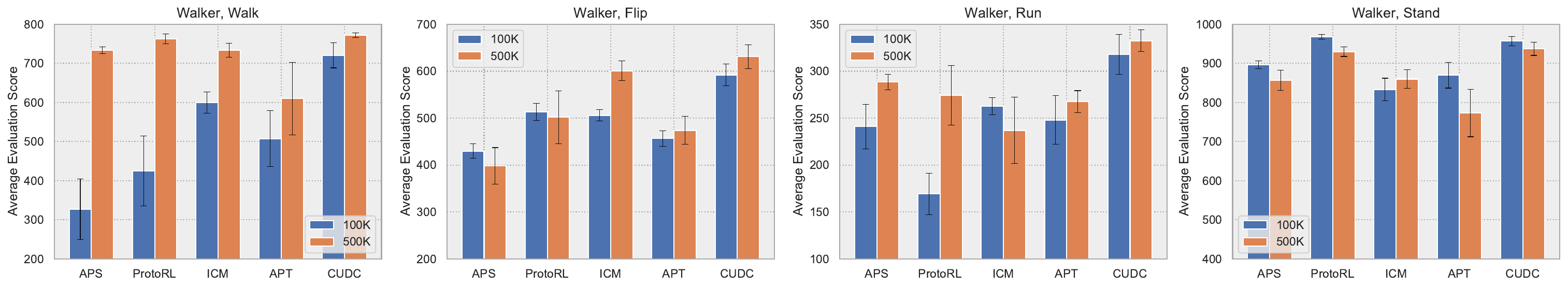} 
  \vskip 0.2in
		\caption{Learning curves (top) and performance scores at 100k and 500K learning steps (bottom) for the CQL agent on 4 downstream tasks of Walker. The proposed CUDC demonstrates superior capability of improving the computational efficiency and learning performances of the CQL agent in all 4 tasks.}
		\label{fig:cql_curve}
        \vskip 0.1in
	\end{figure}

	\subsection{Results of Adding Each Proposed Component on top of Each Other}
	We present the performance scores of the four variants of CUDC at 100K and 500K steps in Table~\ref{table:cudc}. Moreover, the detailed learning curves of the offline RL agent on 4 downstream tasks are shown in Figure \ref{fig:main_curve_full)}. It can be concluded that adapting the temporal distance to reach more distant $k$-step future can substantially improve both computational efficiency and learning capabilities. Moreover, all proposed components facilitated by the proposed reachability module are necessarily important to yield improvement. As a result, the full CUDC model further addressed the instability issue to obtain the minimum standard deviation among all methods.
		\begin{figure}[H]
		\centering
		\includegraphics[width=1\textwidth]{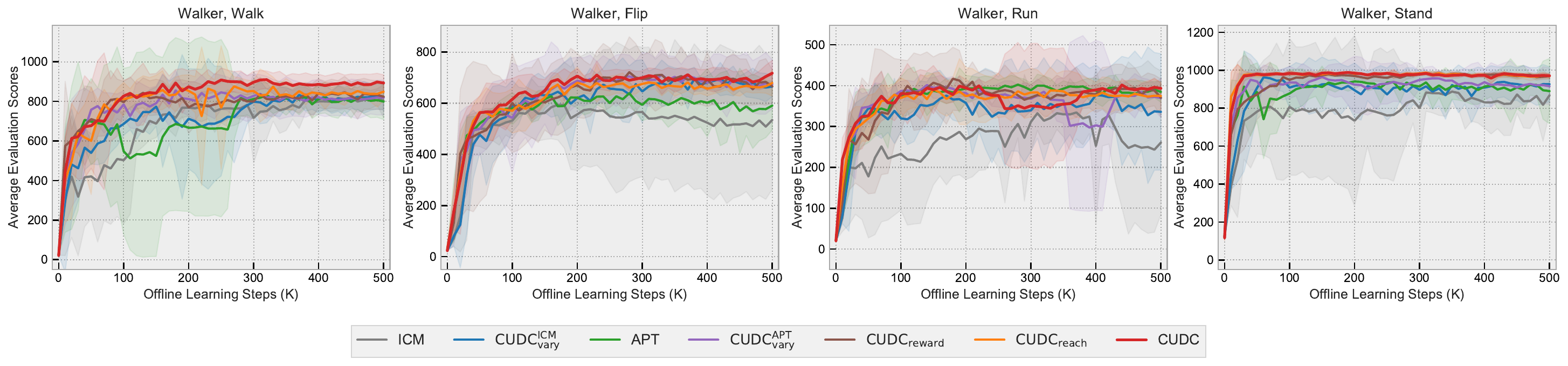} 
		\caption{Full results of learning curves on 4 downstream tasks of Walker environment with the task-agnostic dataset collected by different versions of the proposed CUDC. They are capable of improving both computational efficiency at 100K step and learning capabilities at 500K steps, compared to the baselines of ICM and APT. All proposed components work coherently to collect the high quality dataset for offline learning.}
		\label{fig:variants_curve_full)}
	\end{figure}
	\begin{table*}[h]
		\caption{Performance scores (mean \& standard deviation) of four versions of CUDC at 100K and 500K environment steps. The full model outperforms other versions on 3 out of 4 tasks in computational efficiency (100K) and all four tasks in learning performance (500K) regimes, across 5 random seeds. } \vskip 0.1in
  \centering
		\label{table:cudc} 
            \scalebox{1}{
		\begin{tabular}{llllllll}
			\toprule
			100K Step Score & ICM & $\text{CUDC}_\text{vary}^\text{ICM}$ & APT & $\text{CUDC}_\text{vary}^\text{APT}$ & $\text{CUDC}_\text{reward}$ & $\text{CUDC}_\text{reach}$ & CUDC \\* \midrule
			Walker, Walk    & 503$\pm$258 & 686$\pm$46  & 548$\pm$338 & 785$\pm$56          & 796$\pm$85          & 811$\pm$69          & \textbf{827$\pm$64} \\
			Walker, Flip    & 530$\pm$59  & 589$\pm$82  & 579$\pm$51  & 542$\pm$173         & 601$\pm$156         & 570$\pm$71          & \textbf{615$\pm$66} \\
			Walker, Run     & 233$\pm$111 & 320$\pm$26  & 372$\pm$40  & \textbf{387$\pm$19} & \textbf{385$\pm$69} & 370$\pm$40          & 381$\pm$37          \\
			Walker, Stand   & 797$\pm$312 & 909$\pm$88  & 878$\pm$101 & 936$\pm$69          & 950$\pm$40          & \textbf{983$\pm$11} & \textbf{984$\pm$11} \\ \midrule
			500K Step Score & ICM & $\text{CUDC}_\text{vary}^\text{ICM}$ & APT & $\text{CUDC}_\text{vary}^\text{APT}$ & $\text{CUDC}_\text{reward}$ & $\text{CUDC}_\text{reach}$ & CUDC\\ \midrule
			Walker, Walk    & 802$\pm$67  & 822$\pm$52  & 799$\pm$73  & 820$\pm$71          & 824$\pm$52          & 849$\pm$36          & \textbf{893$\pm$34} \\
			Walker, Flip    & 534$\pm$218 & 665$\pm$32  & 591$\pm$77  & 674$\pm$77          & 666$\pm$53          & 679$\pm$41          & \textbf{717$\pm$41} \\
			Walker, Run     & 261$\pm$142 & 336$\pm$115 & 384$\pm$20  & 370$\pm$38          & \textbf{394$\pm$43} & 375$\pm$27          & \textbf{393$\pm$11} \\
			Walker, Stand &
			868$\pm$73 &
			927$\pm$74 &
			890$\pm$139 &
			916$\pm$78 &
			\textbf{970$\pm$12} &
			\textbf{969$\pm$11} &
			\textbf{971$\pm$7}
			\\*\bottomrule 
		\end{tabular} }
	\end{table*}
	\normalsize

	\subsection{Results of Removing Each Proposed Component from Full Model}
	\label{sec:ablation_remove}
	To further investigate the effectiveness of each proposed component and quantify the importance of them, we carry out additional experiments by respectively removing each proposed component from the full model. In particular, the following models are used to collect the task-agnostic dataset for the Walker environment with 5 random seeds.
	\begin{itemize}
	    \item $\text{CUDC}^{\text{}}_{-\text{Entropy}}$: The proposed intrinsic reward of KNN-based particle entropy of state and action $r_\mathcal{H}(s_t, a_t)$ is removed, and only the prediction error based reward $r_\mathcal{E}(s_t,a_t)$ is used.
	    \item $\text{CUDC}^{}_{-\text{PE}}$: The proposed intrinsic reward of prediction error $r_\mathcal{E}(s_t, a_t)$ is removed, and only the KNN-based particle entropy reward of state and action $r_\mathcal{H}(s_t, a_t)$ is used.
	    \item $\text{CUDC}^{}_{-\text{Regularize}}$: the mechanism of regularizing the backbone DDPG algorithm is removed.
	    \item $\text{CUDC}^{}_{-\text{Vary}}$: $k$-step is fixed to be 3 throughout learning.
	    \item $\text{CUDC}^{}_{-\text{Reach}}$: The proposed reachability module is removed while $k$-step is still adapted through the loss of the dynamics model.
	    \item $\text{CUDC}^{}_{-\text{Inverse}}$: The inverse networks of predicting the action and step $k$ are removed. 
	\end{itemize}
	
	\begin{figure}[H]
		\centering
		\includegraphics[width=1\textwidth]{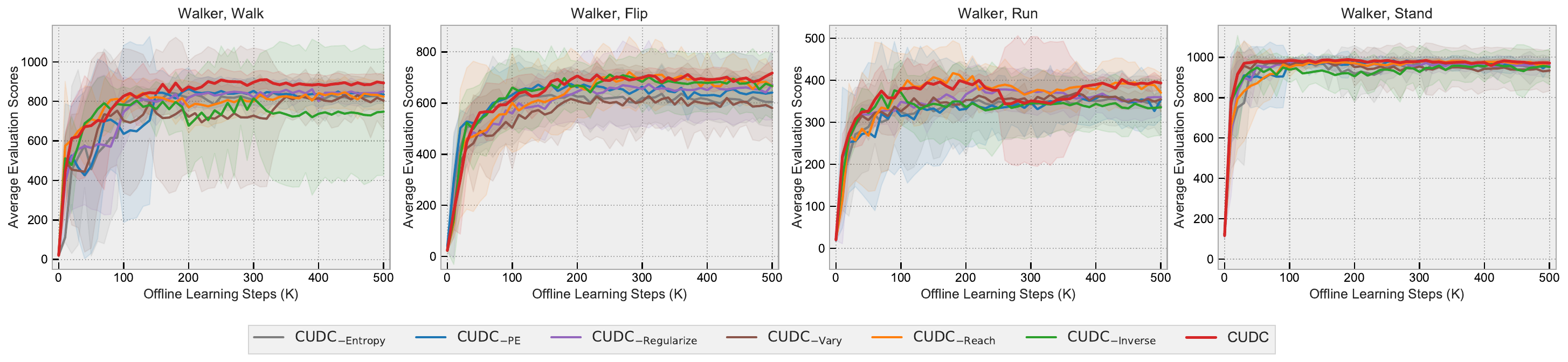} 
		\includegraphics[width=1\textwidth]{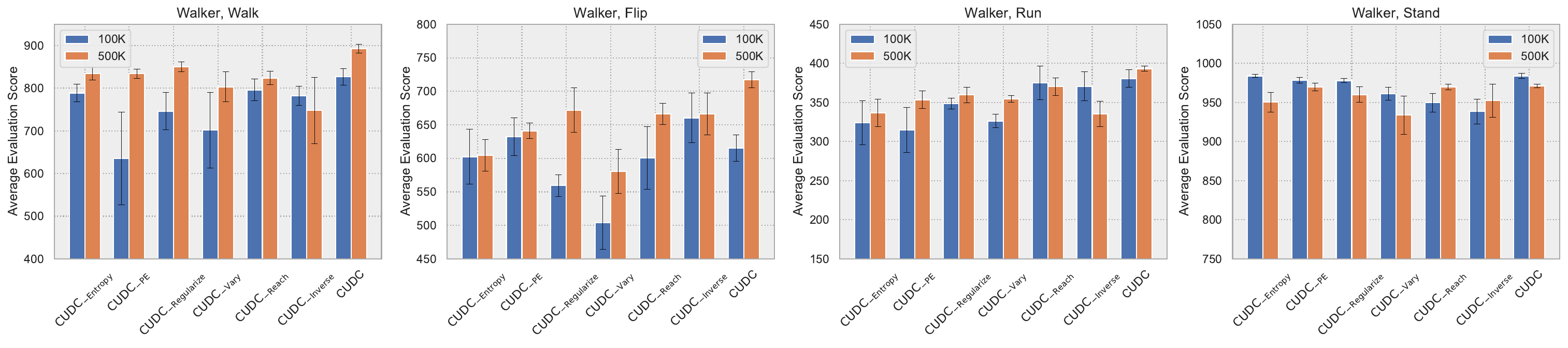} 
		\caption{Learning curves (top) and performance scores at 100K and 500K learning steps (bottom) on 4 downstream tasks of Walker environment with the task-agnostic dataset collected by different models of removing each proposed component from CUDC. Removing any proposed component will cause less desirable performance throughout the offline multi-task learning process.}
		\label{fig:ablation_remove_curve}
	\end{figure}	
		\begin{table}[h]
		\caption{Relative performance of the average scores for the ablation models (removing each proposed component from the full CUDC model) at 100K and 500K learning steps. Adapting k-step is the most effective component to the improved computational efficiency and learning capabilities, while the proposed mixed intrinsic reward is the second.} \vskip 0.1in
		\label{table:remove} 
  \scalebox{1}{
		\begin{tabular}{@{}llllllll@{}}
			\toprule
			100K Relative Performance & $\text{CUDC}_{-\text{Entropy}}$  & $\text{CUDC}_{-\text{PE}}$  & $\text{CUDC}_{-\text{Regularize}}$  & $\text{CUDC}_{-\text{Vary}}$  & $\text{CUDC}_{-\text{Reach}}$  & $\text{CUDC}_{-\text{Inverse}}$  & CUDC \\* \midrule
			Walker, Walk    & 0.954 & \textbf{0.768}  & 0.902 & 0.849         & 0.963          & 0.946         & 1.000 \\
			Walker, Flip    & 0.979 & 1.028 & 0.909 & \textbf{0.819} & 0.976 & 1.074 & 1.000 \\
			Walker, Run     & 0.851 & \textbf{0.827} &  0.916 & 0.857 & 0.986 & 0.974 & 1.000 \\
			Walker, Stand   & 1.000 & 0.995 & 0.994 &0.977 &0.965& \textbf{0.954} & 1.000\\ \midrule
			Mean   & 0.946 & 0.904 & 0.930 &\textbf{0.875} &0.973& 0.987 & 1.000\\ \midrule
			500K Relative Performance & $\text{CUDC}_{-\text{Entropy}}$  & $\text{CUDC}_{-\text{PE}}$  & $\text{CUDC}_{-\text{Regularize}}$  & $\text{CUDC}_{-\text{Vary}}$  & $\text{CUDC}_{-\text{Reach}}$  & $\text{CUDC}_{-\text{Inverse}}$  & CUDC \\ \midrule
			Walker, Walk    & 0.935 & 0.934 & 0.952 & 0.900 & 0.922 & \textbf{0.837} & 1.000\\
			Walker, Flip    & 0.843 & 0.894 & 0.937 & \textbf{0.809} & 0.929 & 0.929 & 1.000 \\
			Walker, Run     & 0.856 & 0.899 & 0.914 & 0.902 & 0.942 & \textbf{0.853} & 1.000 \\
			Walker, Stand & 0.979 & 0.999 & 0.989 & \textbf{0.962} & 0.999 & 0.981 & 1.000
			\\
			\midrule
			Mean   & 0.903 & 0.931 & 0.948 &\textbf{0.893} &0.948& 0.900 & 1.000\\
		\bottomrule 
		\end{tabular} }
	\end{table}
	
	Figure~\ref{fig:ablation_remove_curve} shows the overall offline learning performances. We can observe that limiting $k=3$ and removing the inverse networks will cause the most significant performance decreases among the evaluated models. It implies that adapting the step between current and future states can help to learn rich representation, which plays the most important role in our proposed CUDC to collect high-quality dataset for offline multi-task learning. Meanwhile, the inverse networks are necessary to learn the representation that is robust to the uncontrollable features by the agent’s actions and enables the encoders to capture the dynamics information in the learned representation. 
	
	To quantitatively measure the benefits brought by each proposed component, we compute the relative performance scores at 100K for computational efficiency and 500K for learning capabilities based on the full model. The results are summarized in Table~\ref{table:remove}. For the computational efficiency at 100K steps, removing the $k$-step adaptation has caused the worst performance with a 18.1\% decrease in Flip task and an overall 12.5\% decrease across all 4 downstream tasks, followed by removing prediction-error-based reward (9.6\%), removing regularization (7.0\%) and removing entropy-based reward (5.4\%). For the learning capabilities at 500K steps, removing the $k$-step adaptation has also resulted in the worst performance with a 19.1\% decrease in Flip task and an overall 10.7\% decrease across all 4 downstream tasks, followed by removing inverse networks (10.0\%), removing entropy-based reward (9.7\%) and removing prediction-error-based reward (6.9\%). Thus, we conclude that adapting how many steps into the future that the dynamics model should predict is most effective to the improved computational efficiency and learning capabilities, while the proposed mixed intrinsic reward is the second most effective.

\subsection{Ablation Studies on the Range of Varying $k$-Step}
 In our work, we set the range of varying $k$-step from 3 to 6. It starts from 3 as we follow the same setting as the ExORL benchmark, where all 8 data collection baselines strictly limit $k=3$ for the future state in each transition tuple. As for the threshold of $C_w$ and $C_k$, we did not specifically hypertune these values and they were just set to ensure that $k$ can be varied from 3 to 6 adaptively during the 1M dataset collection process. As for the upper bound of 6, it is an optimal end value to obtain the desired performances. To support this finding, we carry out a sweep of the upper bound of $k$ from 3 to 8, and present the results in Figure~\ref{fig:ablation_k}. Firstly, it can be observed that the learning capabilities at 500K learning steps first increase and then decrease with the increase of the upper bound of $k$, across all 4 downstream tasks. Setting the upper bound of 6 achieves the highest in Walk and Run tasks while it is the second highest in Flip and Stand tasks. Secondly, there is no clear trend of performance at 100K steps by setting different values of the upper bound for $k$. It can be explained by the complexity caused by varying the $k$-step for the future states. However, we can still observe that setting the upper bound of 6 achieves the highest in 3 tasks.
 Therefore, we believe $k$ should vary from 3 to 6 to learn rich representation with more semantic latents.

	\begin{figure}[H]
		\centering
		\includegraphics[width=1\textwidth]{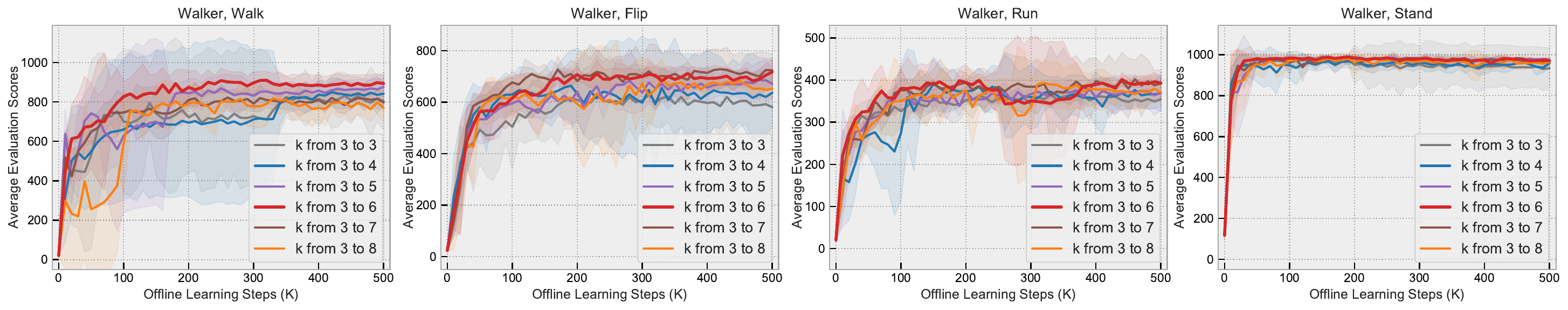} 
		\includegraphics[width=1\textwidth]{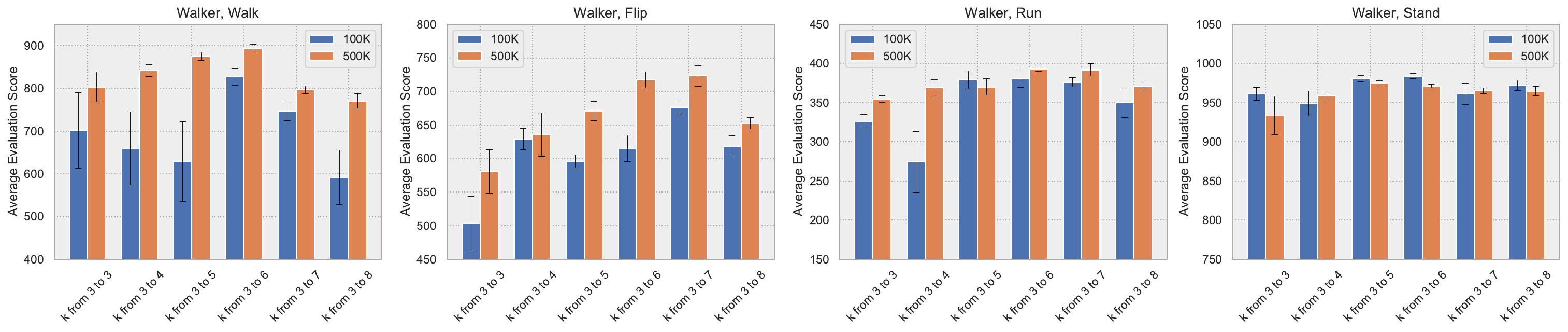} 
		\caption{Learning curves (top) and performance scores at 100K and 500K steps (bottom) on 4 downstream tasks of Walker with the task-agnostic dataset collected by different range of $k$-step. Overall, setting the range from 3 to 6 performs relatively the best. }
		\label{fig:ablation_k}
	\end{figure}

\subsection{Comparison to Other Offline RL Datasets}
While benchmarks such as D4RL \cite{fu2020d4rl} and RL Unplugged \cite{gulcehre2020rl} are commonly used in offline RL, they are not directly applicable to our problem setting. These benchmarks are typically created and evaluated on the same task they were trained on, lacking the capability to assess generalization in unseen tasks. Moreover, most D4RL environments do not support multi-task learning in our context. For instance, datasets collected in a single ant maze task environment cannot be directly used to learn in another ant maze task due to differences in the state space. In contrast, our work focuses on unsupervised data collection to facilitate multi-task offline RL by gathering a single task-agnostic dataset. This presents a distinct challenge compared to other benchmarks in offline RL.

Furthermore, other benchmarks primarily serve to evaluate various offline RL algorithms using pre-collected datasets, which often include human demonstrations, exploratory agents, or hand-coded controllers. While these datasets offer high-quality data, they can be costly and time-consuming to acquire. Expert data collection may not always be feasible or readily available. In contrast, our work does not aim to evaluate different offline RL algorithms, as intended by the D4RL benchmarks. Instead, our focus is on improving the data collection process itself, offering a distinct contribution in the field of offline RL research.

\end{document}